\documentclass{article}
\usepackage{hyperref}
\usepackage{graphicx}
\usepackage[ruled,vlined]{algorithm2e}

\usepackage{amsmath}   
\usepackage{amssymb}   
\usepackage{amsthm} 
\usepackage[numbers]{natbib}
\usepackage{tabularx}
\usepackage{makecell}
\usepackage{booktabs}
\usepackage{multirow}
\usepackage{caption}
\usepackage{array} 
\usepackage{arydshln} 
\usepackage{float} 

\usepackage{geometry}
\usepackage{mathtools}
\usepackage{bm}
\usepackage{xcolor}

\usepackage{graphbox}
\usepackage{enumitem}
\usepackage{titlesec}
\titlespacing*{\paragraph}{0pt}{0.6ex plus 0.2ex minus 0.1ex}{0.5em}
\newcommand{\inlinesubsubsections}{%
  \titleformat{\subsubsection}[runin]
    {\normalfont\normalsize\bfseries}{\thesubsubsection}{0.75em}{}[.]%
  \titlespacing*{\subsubsection}{0pt}{0.9ex plus 0.2ex minus 0.1ex}{0.6em}}
\newcommand{\displaysubsubsections}{%
  \titleformat{\subsubsection}
    {\normalfont\normalsize\bfseries}{\thesubsubsection}{1em}{}%
  \titlespacing*{\subsubsection}{0pt}{3.25ex plus 1ex minus .2ex}{1.5ex plus .2ex}}
\usepackage{changepage}

\newcommand{\argmin}[1]{\underset{#1}{\operatorname{arg}\,\operatorname{min}}\;}

\newtheorem{theorem}{Theorem}[section]
\newtheorem{corollary}{Corollary}[section]

\newtheorem{proposition}[theorem]{Proposition}
\newtheorem{definition}{Definition}[section]

\newtheorem{proposition*}[theorem]{Proposition}

\renewcommand{\arraystretch}{1.2} 

\newcommand{\lb}{\left(}
\newcommand{\rb}{\right)}

\DeclareMathOperator{\diag}{diag}

\newcommand{\mycomment}[1]{}

\title{Low-Rank Dependence Decomposition via Accelerated\\ Symmetric Non-negative Matrix Factorization}
\author{\normalsize Lavinia Ghita\footnote{Correspondence to \texttt{lghita@nvidia.com}}, Dhruv Desai, Jake Goldberg, Roman Yokunda Enzmann}

\date{NVIDIA}

\makeatletter
\def\@maketitle{%
  \newpage
  \null
  \vskip 0.5em%
  \begin{center}%
  \let \footnote \thanks
    {\LARGE \@title \par}%
    \vskip 1.5em%
    {\large
      \lineskip .5em%
      \begin{tabular}[t]{c}%
        \@author
      \end{tabular}\par}%
    \vskip 1em%
    {\large \@date}%
  \end{center}%
  \par
  \vskip 1.5em}
\makeatother

\begin{document}

\maketitle

\renewenvironment{abstract}
  {\small\begin{center}\bfseries Abstract\end{center}\begin{adjustwidth}{1.88cm}{1.88cm}\noindent}
  {\end{adjustwidth}}

\begin{abstract}
Symmetric non-negative matrix factorization (SymNMF) recovers latent group structure from a dependence matrix, but its dense, quadratic-memory objective has confined prior work to moderate sizes. We present a large-scale GPU study of seven algorithm families (over 30 configurations) on absolute Pearson correlation and tail pairwise dependence matrices from Extreme Value Theory, two proxies for empirical risk-factor estimation on large portfolios. A trace-identity reformulation eliminates all $n \times n$ intermediates, so a single GPU reaches $n \approx 10^5$ and multi-node distribution scales to $n = 10^6$ and beyond. Under a two-phase protocol, eleven methods converge at moderate scale; six remain efficient enough at $n = 10^5$ (five AdaGrad-family plus ADMM), and five AdaGrad-family methods still converge at $n = 10^6$: AdaGrad, RMSprop, and three we introduce (Piecewise AdaGrad, Row-Stochastic SVRG, Block-SVRG AdaptGrow). At $n = 10^6$ the fastest solver tracks the matrix spectrum: Block-SVRG AdaptGrow wins on the flat, ill-conditioned tail-dependence spectrum, where its lower per-iteration cost decides a long factorization, and full-batch AdaGrad wins on the dominant-low-rank correlation spectrum, where the run is short. We also benchmark spherical K-means as a hard-label baseline: cheaper when angular cluster structure is present, yet provably degenerate once the matrix collapses toward a single common factor, where the soft factorization remains necessary.
\end{abstract}

\section{Introduction}

Given a symmetric matrix $S \in \mathbb{R}_+^{n \times n}$, symmetric non-negative matrix factorization (SymNMF) seeks $H \in \mathbb{R}_+^{n \times k}$ with $k \ll n$ such that $S \approx HH^\top$, revealing latent groups whose members share common dependence structure; $H$ serves as both a soft clustering and a learned embedding. We study two input types: absolute Pearson correlation matrices, capturing dependence across the full distribution, and tail pairwise dependence matrices (TPDM) \cite{Cooley2019}, grounded in Extreme Value Theory \cite{Coles2001,resnick2010heavy}, capturing dependence conditional on extreme events. In financial risk, such factor models are a core tool: correlation-based factors summarize broad market dynamics, while TPDM-based factors isolate the co-dependencies that drive systemic losses. The same pair also spans complementary spectra for the solver study, from correlation matrices with a clean signal--noise eigenvalue gap to tail matrices whose mass concentrates in a near-dominant common factor, so comparisons are not tied to a single landscape (Section~\ref{snmf:spectral}). Realistic portfolios span thousands to millions of instruments, so both need a factorization that scales.

Although SymNMF is well-established as a factorization and clustering problem \cite{kuang2012symmetric,kuang2015symnmf}, practical deployment at scale faces three gaps. First, there is no systematic comparison of SymNMF solver families on realistic inputs with controlled ground truth, and the $O(n^2)$ memory of the dense objective has kept prior work at moderate sizes. Second, the interaction between constraint handling (projection vs.\ softplus) and adaptive optimizers is unexplored. Third, methods that work at moderate scale still struggle on the long runs that large-$n$ SymNMF requires: adaptive steps either forget useful history or saturate as accumulators grow.

The paper addresses each gap and yields a recommendation: which solver to use for a given input type and scale.
First, we present a systematic GPU evaluation of seven algorithm families (over 30 configurations) on correlation and TPDM inputs from $n = 10^2$ to $10^6$. A trace-identity reformulation eliminates all $n \times n$ intermediates, so a single GPU handles $n$ up to ${\sim}10^5$ and multi-node distribution scales to $n = 10^6$ and beyond. At large scale a small set of adaptive first-order methods dominates the reliability--efficiency frontier, with the fastest depending on the matrix type (Section~\ref{results:phase2}).
Second, we introduce softplus reparameterization of the non-negativity constraint and characterize when it helps: it aids quasi-Newton at small scale but corrupts long-memory adaptive accumulators as $n$ grows, so the large-scale campaign runs in projected space.
Third, motivated by that tension, we use diagonal AdaGrad \cite{duchi2011adagrad}, whose full-history accumulator suits SymNMF's stationary objective but saturates on long runs, and propose three empirical extensions. \emph{Piecewise AdaGrad} resets the accumulator on objective stagnation and recalibrates per entry to the current gradient scale; \emph{Row-Stochastic SVRG} and \emph{Block-SVRG AdaptGrow} adapt variance reduction and adaptive batching to SymNMF to cut per-iteration cost while retaining universal convergence (Section~\ref{snmf:perf}).
Additionally, we connect orthogonality-constrained SymNMF to spherical $K$-means \cite{ding2005,kuang2015symnmf} and apply it on the rows of~$S$ as the practical discrete baseline when only cluster assignments are needed. At scale it is markedly cheaper when angular structure is present, but provably degenerate under a dominant common factor, where soft factorization remains necessary.

The remainder is organized as follows: Section~\ref{snmf} formulates the SymNMF problem and surveys solver families; Section~\ref{maths} constructs the two input types (correlation and TPDM) and their shared low-rank structure; Section~\ref{methodology} details the methodology; Section~\ref{results} presents Phase~1 ($n \le 10^4$) and Phase~2 ($n \ge 10^5$) results; Section~\ref{clustering} covers hard clustering alternatives; Section~\ref{related} reviews related work; Section~\ref{conclusions} concludes. Proofs and supplementary material are in the Appendix.

\section{Symmetric Non-negative Matrix Factorization}\label{snmf}

This section formulates the factorization and the solvers we benchmark: the objective and identifiability below, a memory-efficient loss in Section~\ref{snmf:loss}, seven algorithm families in Section~\ref{snmf:algos}, and projection versus softplus in Section~\ref{snmf:logspace}.

Symmetric non-negative matrix factorization (SymNMF) \cite{kuang2012symmetric,kuang2015symnmf} specializes NMF \cite{lee1999learning,lee2001algorithms} to the symmetric case. Given an $n \times n$ symmetric, entrywise non-negative dependence matrix~$S$ (Section~\ref{maths}), we seek
\begin{equation}\label{eq:symnmf}
\min_{H \geq 0}\, f(H),\qquad f(H) = \left\Vert S - HH^\top \right\Vert ^2_F,
\end{equation}
where $H \in \mathbb{R}_+^{n \times k}$ with $k \ll n$.
Figure~\ref{fig:symnmf} illustrates the decomposition for a small example.

\begin{figure}[ht]
    \centering
    \includegraphics[width=0.83\textwidth]{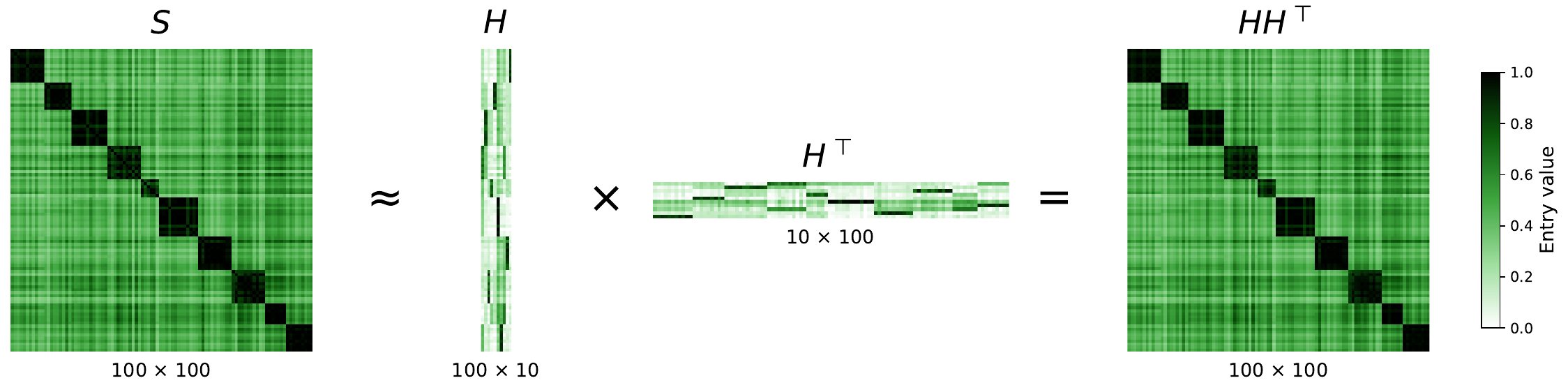}
    \caption{SymNMF on a $100 \times 100$ benchmark dependence matrix ($k = 10$). Left: input $S$; right: reconstruction $HH^\top$ via $H \in \mathbb{R}_+^{100 \times 10}$. The block structure of $S$ is recovered in the columns of $H$ (relative Frobenius residual in this example $\|S - HH^\top\|_F / \|S\|_F = 1.6\%$).}
    \label{fig:symnmf}
\end{figure}

\paragraph{Identifiability.}
The SymNMF factorization is not unique: if $H^*$ is a solution, then so is $H^* \Pi$ for any $k \times k$ permutation matrix $\Pi$, since $(H^*\Pi)(H^*\Pi)^\top = H^* \Pi \Pi^\top H^{*\top} = H^* H^{*\top}$.
In the unconstrained case $S = BB^\top$, $B$ is determined only up to right-multiplication by an orthogonal matrix $Q$; non-negativity $H \ge 0$ removes this continuous rotational ambiguity and leaves only discrete column permutation \cite{kuang2015symnmf,gillis2020nmf}.
Under mild separability conditions on $S$ (a ``sufficiently scattered'' factor structure), the non-negative rank-$k$ factorization is essentially unique up to that permutation \cite{gillis2020nmf}. Unlike asymmetric NMF, there is no diagonal-scaling freedom: $H \mapsto H\,\mathrm{diag}(\bm{c})$ alters $HH^\top$.

\subsection{Memory-Efficient Loss Reformulation}\label{snmf:loss}

A naive implementation forms the \(n \times n\) residual \(S - HH^\top\) explicitly; including intermediate buffers, this requires ${\sim}\,5$ dense $n \times n$ matrices ($S$, the product $HH^\top$, the residual, and backward-pass temporaries; $20\,n^2$ bytes in fp32), whereas the problem intrinsically stores only $S$ itself and the $n \times k$ factor ($n^2 + nk$ entries).

We reformulate $f$ by expanding the squared Frobenius norm via the trace identity:
\begin{equation}\label{eq:trace_loss}
\|S - HH^\top\|_F^2 = \|S\|_F^2 - 2\,\mathrm{tr}(H^\top S H) + \|H^\top H\|_F^2.
\end{equation}
All intermediates in~\eqref{eq:trace_loss} and~$\nabla_H f$ are at most \(n \times k\) or \(k \times k\). The gradient is
\[
\nabla_H f = 4\bigl(H(H^\top H) - SH\bigr),
\]
evaluated without $n \times n$ temporaries: $H(H^\top H)$ is formed right-to-left via a $k \times k$ then an $n \times k$ product.
Gradient-based methods use $\nabla_H f$ directly; multiplicative updates reuse $SH$ and $H(H^\top H)$ as a ratio. The alternating direction method of multipliers (ADMM) solves a $k \times k$ linear system and never forms $\nabla_H f$ (Appendix~\ref{algo_detail}).

\begin{table}[H]
\centering\small
\caption{Peak memory and single-GPU capacity of the naive versus trace-reformulated objective (fp32). Removing all $n \times n$ intermediates roughly doubles the largest feasible $n$.}
\label{tab:memory}
\smallskip
\begin{tabular}{lccc}
\toprule
\textbf{Implementation} & \textbf{Peak memory} & \textbf{Max $n$ (80\,GB)} & \textbf{Max $n$ (192\,GB)} \\
\midrule
Naive ($n{\times}n$ residual) & ${\sim}20\,n^2$ B & ${\sim}60{,}000$ & ${\sim}95{,}000$ \\
Reformulated (trace) & ${\sim}4\,n^2$ B & ${\sim}130{,}000$ & ${\sim}210{,}000$ \\
\bottomrule
\end{tabular}
\end{table}

This purely algebraic change approximately doubles the maximum~$n$ on a single GPU (Table~\ref{tab:memory}). The reformulation is applied uniformly across all solvers.

\paragraph{Numerical consideration.} When the normalized squared error $E_t = \|S - HH^\top\|_F^2 / \|S\|_F^2$ is small, the three terms in \eqref{eq:trace_loss} are each of size $\approx \|S\|_F^2$ and nearly cancel, amplifying rounding error. The relative error in the computed loss scales as ${\sim}\,2\varepsilon / E_t$, where $\varepsilon \approx 1.2 \times 10^{-7}$ is fp32 machine epsilon. At our target $E_t \sim 10^{-4}$ this is only ${\sim}\,0.2\%$, negligible for algorithm comparison (fp64 accumulation would matter only at tighter targets). The effect is worst in the near-exact regime (TPDM at $n = 10^6$, where $E_t \to 0$), and the loss itself becomes precision-limited. We therefore certify convergence with the projected-gradient gate of Section~\ref{snmf:convergence}, not the loss value. Half-precision formats (fp16, bf16, $\varepsilon \sim 10^{-3}$) would make the cancellation comparable to $E_t$ itself at $10^{-4}$, rendering the loss uninformative; whether fp32 accumulation of the three trace terms could rescue an fp16/bf16 $O(n^2 k)$ GEMM is left open. The reported experiments run in fp32; as a check we also ran the same solves with TF32 Tensor-Core multiplies (10-bit mantissa, fp32 accumulate) and obtained unchanged converged $E_t$ (Section~\ref{snmf:compute}).

\subsection{Algorithm Families}\label{snmf:algos}

We evaluate seven algorithm families under three paradigms for $\min_{H \ge 0} \|S - HH^\top\|_F^2$:
\begin{enumerate}[leftmargin=2em,topsep=3pt,itemsep=1pt]
\item[\textbf{(I)}] \textbf{Gradient projection.}  Compute $\nabla_H f$ and enforce $H \ge 0$ by clamp projection or softplus reparameterization (Section~\ref{snmf:logspace}). Families~2, 3, and~5 span first-order, adaptive, and second-order gradient information.
\item[\textbf{(II)}] \textbf{Mirror descent.}  Replace Euclidean geometry with a negative-entropy Bregman divergence so non-negativity is intrinsic to the mirror map. Family~1 comprises multiplicative updates (MU), exponentiated gradient, and trace-norm regularized MU.
\item[\textbf{(III)}] \textbf{Lagrangian relaxation / splitting.}  Decouple bilinearity by variable splitting or block-coordinate subproblems. Families~4 and~6 cover block coordinate descent (BCD), the alternating direction method of multipliers (ADMM), and a randomized surrogate.
\end{enumerate}
Family~7 deep-unfolds an MU-like iteration end-to-end. In total we obtain over 30 solver configurations and over 100 hyperparameter runs; update rules and costs are in Appendix~\ref{algo_detail}. Table~\ref{tab:families} summarizes the families, detailed next.\footnote{As of 2026, no mainstream GPU library provides GPU-accelerated SymNMF; we implement all solvers in PyTorch, calling cuBLAS via \texttt{torch.linalg}, and the hard-clustering baseline (Section~\ref{clustering}) in the same framework.}

The objective $f$ is quartic in the entries of $H$ and globally non-convex. The two-factor surrogate $\|S - HW^\top\|_F^2$ is biconvex in $(H,W)$: with either factor fixed, the other subproblem is convex least squares, and SymNMF is recovered by the coupling $W = H$. Block-coordinate methods (Family~4) and ADMM (Family~6) exploit this lifting, ADMM enforcing $W = H$ by an augmented Lagrangian at the cost of slow, coupling-limited convergence \cite{boyd2011distributed}. Adaptive methods (Family~3) stay with the single-factor objective. All first-order methods are built from the products $SH$ and $H(H^\top H)$ of Section~\ref{snmf:loss}, at cost $O(n^2 k + nk^2)$ per iteration with no $n \times n$ intermediate.

\begin{table}[t]
\centering\footnotesize
\setlength{\tabcolsep}{4pt}\renewcommand{\arraystretch}{1.12}
\caption{Solver families. Update rules and costs: Appendix~\ref{algo_detail}.}
\label{tab:families}
\begin{tabularx}{0.92\textwidth}{@{}c >{\raggedright\arraybackslash}p{2.8cm} >{\raggedright\arraybackslash}p{4.8cm} >{\raggedright\arraybackslash}X c@{}}
\toprule
\textbf{\#} & \textbf{Family} & \textbf{Methods} & \textbf{Role} & \textbf{\S} \\
\midrule
1 & Multiplicative / mirror descent & MU, exp.\ gradient, trace-norm MU & Intrinsic $H\ge 0$; slow first-order & \ref{fam:1} \\
2 & Projected first-order & PGD+momentum, APG/FISTA & Euclidean projection onto $H\ge 0$ & \ref{fam:2} \\
3 & Adaptive first-order & AdaGrad, RMSprop, Adam/NAdam/Adan, extensions & Per-element adaptive rates & \ref{fam:3} \\
4 & Block coordinate & HALS, ANLS & Column / two-block least squares & \ref{fam:4} \\
5 & Second-order & Diag.\ Newton, L-BFGS, PGNCG & Curvature-based local steps & \ref{fam:5} \\
6 & Splitting & ADMM, rand.\ LAI-SymNMF & Split bilinear constraint & \ref{fam:6} \\
7 & Deep unfolding & SymNMF-Net & Unroll MU; learn block parameters & \ref{fam:7} \\
\bottomrule
\end{tabularx}
\end{table}

\inlinesubsubsections
\subsubsection{Family 1: Multiplicative / Mirror Descent}\label{fam:1}
The dampened MU rule of \cite{he2011symmetric} updates $H$ via element-wise scaling:
\begin{equation}\label{eq:mu_update}
    H \leftarrow \tfrac{1}{2}\,H \odot \left(\mathbf{1} + SH \oslash \bigl(H(H^\top H) + \epsilon\bigr)\right),
\end{equation}
guaranteeing monotonic decrease and preserving non-negativity by construction.  We additionally test a \textbf{trace-norm regularized} variant that adds a penalty $\lambda(\operatorname{tr}(S) - \|H\|_F^2)^2$, approximating the scale-stabilizing effect of the unnormalized KL divergence at $O(nk)$ cost (no $n \times n$ intermediates).  \textbf{Mirror descent} with the negative-entropy mirror map yields the exponentiated gradient update $H \leftarrow H \odot \exp(-\eta\,\nabla_H f)$, which preserves strict positivity without projection \cite{nemirovsky1983problem}.  All three are slow first-order methods and sensitive to initialization.

\subsubsection{Family 2: First-Order Gradient Methods}\label{fam:2}
Projected Gradient Descent (PGD) with Armijo backtracking \cite{armijo1966minimization} and heavy-ball momentum \cite{polyak1964some}, and Accelerated Proximal Gradient (APG/FISTA) with adaptive restart \cite{beck2009fast,odonoghue2015adaptive}.

\subsubsection{Family 3: Adaptive First-Order Methods}\label{fam:3}
Per-element adaptive learning rates. Throughout this paper, \textbf{AdaGrad} means diagonal AdaGrad \cite{duchi2011adagrad}: projected SGD with accumulator $G \leftarrow G + g \odot g$ and step $H \leftarrow [H - \eta\, g / \sqrt{G}]_+$ (equivalently $G_t = \sum_{s\le t} g_s^2$ entrywise). The original full-matrix preconditioner is $(nk)\times(nk)$ and intractable here; we never use it. The same diagonal projected template underlies \textbf{RMSprop} \cite{hinton2012rmsprop} (EMA of squared gradients), \textbf{Adam} \cite{kingma2014adam} (first- and second-moment EMAs), \textbf{NAdam} \cite{dozat2016incorporating} (Nesterov-accelerated Adam), and \textbf{Adan} \cite{xie2023adan} (three-term Nesterov momentum with adaptive rates). We also test AdaGrad variants aimed at accumulator staleness and cheaper steps (Section~\ref{snmf:perf}): \textbf{Piecewise AdaGrad} (stagnation-triggered accumulator resets); \textbf{Row-Stochastic AdaGrad} (row subsampling, $O(|I|\,n\,k)$ per step); \textbf{Row-Stochastic SVRG} (periodic full-gradient snapshots \cite{johnson2013svrg}); \textbf{Adaptive Growing-Batch AdaGrad} (\textsc{AdaptGrow}; entry sampling at $O(mk)$ with adaptive batch growth \cite{friedlander2012hybrid, byrd2012sample}); and \textbf{Block-SVRG \textsc{AdaptGrow}} (dense $|I|\times|J|$ blocks for cuBLAS throughput, with an SVRG hybrid snapshot against per-row accumulator heterogeneity). These are empirical engineering methods: we report wall-clock and KKT reliability under Section~\ref{methodology}, not regret or rate guarantees. Kronecker-factored preconditioners were tested but excluded (Appendix~\ref{algo_detail}).

\subsubsection{Family 4: Block Coordinate Descent}\label{fam:4}
\textbf{Hierarchical alternating least squares (HALS)} \cite{kimpark2008} updates columns of $H$ individually via closed-form least-squares subproblems.  \textbf{Alternating nonnegative least squares (ANLS)} \cite{kim2014algorithms} uses a coupled two-factor formulation, solving $k \times k$ Cholesky systems per iteration.

\subsubsection{Family 5: Second-Order Methods}\label{fam:5}
A \textbf{diagonal-Newton} baseline with Levenberg--Marquardt damping \cite{levenberg1944method,marquardt1963algorithm} and Armijo backtracking \cite{armijo1966minimization} on the Hessian diagonal; projected \textbf{L-BFGS} \cite{nocedal2006numerical}; and \textbf{projected Gauss--Newton with truncated conjugate gradient (PGNCG)} \cite{eswar2020pgncg}, the local solver underlying the distributed PLANC library \cite{eswar2021planc}.

\subsubsection{Family 6: Splitting Methods}\label{fam:6}
\textbf{ADMM} \cite{boyd2011distributed} decouples the problem via variable splitting ($W = H$, penalty $\rho$), alternating between a least-squares $H$-update, a clamped $W$-update, and dual ascent.  \textbf{Randomized LAI-SymNMF} \cite{hayashi2024randomized} solves a low-rank eigendecomposition surrogate via cheap HALS-style updates.

\subsubsection{Family 7: Deep Unfolding}\label{fam:7}
\textbf{SymNMF-Net} \cite{li2022symncf} unfolds $T$ iterations of a multiplicative-update rule into a feedforward network whose per-block parameters (step sizes, damping) are learned end-to-end via backpropagation through the unrolled computation graph. The method requires per-instance training: the network weights are optimized for each input matrix $S$ individually, since a single set of weights does not generalize across matrices of different spectral structure.
\displaysubsubsections

\subsection{Constraint Handling Strategies}\label{snmf:logspace}

Every factorization must satisfy \(H \ge 0\). Family~1 enforces this intrinsically; Families~4 and~6 build it into their subproblems; Family~7 uses a ReLU. Gradient-projection families (2, 3, and~5) are compared under clamp projection and softplus.

\paragraph{Projected space.} The optimizer updates \(H\) directly and clamps after each step, including adaptive and second-order updates:
\[
H^{(t+1)} = \max\!\bigl(H^{(t)} - \eta\,\nabla_H f,\; 0\bigr).
\]
There is no reparameterization overhead, but once an entry reaches \(H_{ij}=0\) it may remain there: under non-convexity the landscape can admit boundary local minima that are stable to single-coordinate moves yet escapable by a coordinated interior step. Softplus avoids this trap.

\paragraph{Softplus reparameterization (LogSpace).} We introduce an unconstrained parameter \(\Theta \in \mathbb{R}^{n \times k}\) and set \(H = \mathrm{softplus}(\Theta) = \log(1 + e^\Theta)\) entrywise \cite{subramani2024rethinking}, distinct from the ADMM factor \(W\). Softplus is strictly positive, so \(H > 0\) without projection. The optimizer runs on \(\Theta\); gradients chain-rule as
\begin{equation}\label{eq:logspace_grad}
\nabla_\Theta f = \nabla_H f \;\odot\; \sigma(\Theta), \qquad \sigma(\Theta) = \frac{1}{1 + e^{-\Theta}}.
\end{equation}
Because \(\sigma(\Theta_t)\) depends on the iterate, the \(\Theta\)-space gradient is non-stationary even when the \(H\)-space landscape is not.

\paragraph{Interaction with optimizer memory.} Whether adaptive second-moment statistics track \(\sigma(\Theta_t)\) depends on memory length, which predicts (Section~\ref{results:logspace}):
\begin{itemize}[leftmargin=2em,topsep=2pt,itemsep=1pt]
\item \textbf{Short-memory} (RMSprop, \(\alpha = 0.99\), effective window \({\sim}100\) iterations): the EMA \(v_t = \alpha\, v_{t-1} + (1-\alpha)\, g_t^2\) forgets stale magnitudes fast enough to track \(\sigma(\Theta_t)\).
\item \textbf{Long-memory} (Adam, \(\beta_2 = 0.999\), window \({\sim}1{,}000\); AdaGrad, unbounded cumulative sum): statistics mix regimes where \(\sigma(\Theta)\) differed, so adaptive rates become mis-scaled.
\item \textbf{No second-moment accumulator} (diagonal Newton, L-BFGS): local curvature is not corrupted by \(\sigma\)-drift; the smoother, projection-free surface can even help at small scale.
\end{itemize}
Softplus is therefore expected to help only at small scale, and to hurt long-memory adaptive methods as runs lengthen.

\section{Dependence Matrices}\label{maths}

Section~\ref{snmf} defined the SymNMF machinery that factorizes a symmetric, entrywise non-negative matrix $S$; this section specifies the two concrete dependence matrices that serve as that input, together with the low-rank structure they share: Section~\ref{corr_intro} recalls the (absolute) correlation matrix, capturing dependence across the full distribution; Sections~\ref{regvar}--\ref{tpdm} build the tail pairwise dependence matrix (TPDM) from multivariate regular variation, capturing dependence conditional on extreme events; Section~\ref{estimation} gives its empirical estimator; and Section~\ref{decompositions} states the low-rank decomposition both types admit.

\subsection{Correlation Matrices}\label{corr_intro}

Under a linear factor model $\bm{X} = A\bm{Z} + \bm{\varepsilon}$ with $\bm{Z} \sim \mathcal{N}(0, I_{k+1})$ and $\bm{\varepsilon} \sim \mathcal{N}(0, \sigma_\varepsilon^2 I_n)$, one has $\operatorname{Cov}(\bm{X}) = AA^\top + \sigma_\varepsilon^2 I$ and population correlation
\[
C_{ij} = \frac{(AA^\top)_{ij}}{\sqrt{(AA^\top)_{ii} + \sigma_\varepsilon^2}\;\sqrt{(AA^\top)_{jj} + \sigma_\varepsilon^2}},
\]
with unit diagonal. The matrix $C$ is symmetric positive semidefinite with entries in $[-1,1]$. Our benchmarks use the absolute matrix $|C|$, which lies in $[0,1]$ but need not remain PSD at large~$n$ (Section~\ref{snmf:spectral}); it summarizes pairwise linear dependence over the full distribution of~$\bm{X}$.

In applications where dependence during normal conditions differs from dependence during extreme events \cite{Cooley2019}, the correlation matrix may fail to capture the co-movement structure that drives systemic risk. This motivates the construction of a matrix that isolates tail dependence, which we develop next using the theory of multivariate regular variation.

\subsection{Regular Variation and Max-Stable Distributions}\label{regvar}

The TPDM of Section~\ref{tpdm} is defined for multivariate regularly varying vectors, which sit in the heavy-tailed (Fr\'echet, $\xi > 0$) domain of attraction; the Weibull ($\xi < 0$, finite upper endpoint) and Gumbel ($\xi = 0$) domains fall outside that regularly varying setup, so the construction does not apply there (\cite{Coles2001}, Theorem~3.1; \cite{Embrechts2013ExtremeValueTheoryFinanceInsurance}). A survival function with tail index $\alpha>0$ satisfies $\overline{F}(tx)/\overline{F}(t)\to x^{-\alpha}$ as $t\to\infty$ (i.e.\ $\overline{F}\in RV_{-\alpha}$). Larger $\alpha$ is a lighter power-law tail. More generally:

\begin{definition} (\cite{resnick2010heavy}) A measurable function $U:\mathbb{R}_{+} \rightarrow \mathbb{R}_{+}$ is regularly varying at $\infty$ with index $\alpha \in \mathbb{R}$ (written $U \in RV_{\alpha}$) if for all $x > 0$, $$\lim_{t \rightarrow \infty} \frac{U(tx)}{U(t)} = x^{\alpha} \text{.}$$
\end{definition}

The multivariate extension proceeds as follows. For $n$-dimensional i.i.d.\ random vectors, the componentwise maxima $\bm{M}_N = (\max_{i \le N} X_{i,1}, \dots, \max_{i \le N} X_{i,n})$ may converge after renormalization to a multivariate extreme value (MEV) distribution~$G$. Using a standard transformation (\cite{de2010extreme}, \cite{resnick19872014extreme}), the marginals can be standardized to unit Fr\'echet, separating the dependence problem from marginal behavior. Under this standardization,
\begin{equation}\label{max_regvar}
G\lb \bm{x}\rb = \exp(-V(\bm{x})) \text{ with } V(\bm{x}) = \int_{\mathbb{S}_{+}^{n-1}} \max_{i=1,\dots,n} \frac{w_i}{x_i} d\mathcal{H}(\bm{w}),
\end{equation}
where $\mathcal{H} = \mathcal{H}_{\bm{X}}$ is the spectral (angular) measure on the positive orthant of the unit sphere $\mathbb{S}_{+}^{n-1}$, normalized so that $\int_{\mathbb{S}_{+}^{n-1}} w_i \, d\mathcal{H}(\bm{w}) = 1$ for each $i$.\footnote{Classical EVT notation writes $H$ for the angular measure and $\Theta_{n-1}$ for this sphere~\cite{resnick2010heavy,Cooley2019}; we write $\mathcal{H}$ and $\mathbb{S}_{+}^{n-1}$ to reserve $H$ for the SymNMF factor and $\Theta$ for the softplus reparameterization of Section~\ref{snmf:logspace}.} 

\begin{definition} (\cite{resnick2010heavy}) A random vector $\bm{X} \in \mathbb{R}_{+}^{n}$ is regularly varying if there exists $\{ b_N \}_{N \geq 1} \subset \mathbb{R}$ and a limit measure $\nu_{\bm{X}}$ on $[0,\infty]^n \backslash \{0\}$ such that 
$$N \mathbb{P} \left( \frac{ \bm{X}}{b_N} \in \cdot \right) \xrightarrow{v} \nu_{\bm{X} } (\cdot) \text{, as } N \rightarrow \infty \text{,}$$
where $\xrightarrow{v}$ denotes vague convergence. It can be shown that $b_N = L(N) N^{1/\alpha}$ for a slowly varying $L$ and tail index $\alpha > 0$. Denote $RV_{+}^n(\alpha)$ the set of $n$-dimensional regularly varying vectors with tail index $\alpha$. 
\end{definition}

\paragraph{Polar decomposition.}\label{polar} The limiting measure $\nu_{\bm{X}}$ has the scaling property $\nu_{\bm{X}} (aC) = a^{-\alpha} \nu_{\bm{X}}(C)$ (\cite{resnick2010heavy}). In polar coordinates $(R, \bm{W}) = (\|\bm{X}\|_2, \bm{X}/\|\bm{X}\|_2)$ this factorizes as
$$\nu_{\bm{X}} (dr \times d\bm{w}) = \alpha r^{-\alpha -1} dr\, d\mathcal{H}_{\bm{X}}(\bm{w}) \text{,}$$
separating the radial component (determined by~$\alpha$) from the angular component (the spectral measure~$\mathcal{H}$). Since $\alpha$ is fixed by the marginal standardization, only $\mathcal{H}$ remains to be estimated.

\subsection{Tail Pairwise Dependence Matrix}\label{tpdm}

We adopt the TPDM construction of \cite{Cooley2019}, which constructs a symmetric pairwise matrix describing the tail dependence of multivariate regularly varying random vectors. Using tail dependence rather than the full distribution is motivated by regulatory frameworks \cite{BCBS2011} and active portfolio construction \cite{DeLuca2011Tail,Lohre2020Hierarchical}. We consider one-sided (lower-tail) extremes, motivated by the asymmetry of financial risk management applications. The regular-variation lemmas for the linear factor model used below are stated and proved in Appendix~\ref{proofs}.

Under a linear factor model, regular variation propagates by two elementary rules (Propositions~\ref{prop:sum}--\ref{prop:prod}): independent regularly varying summands with a common normalizing sequence have limit measure equal to the sum of their limit measures, and scaling a regularly varying vector by $a>0$ multiplies its limit measure by $a^\alpha$. Applied to $A\bm{Z} = \sum_{j=1}^{k} \bm{a}_j Z_j$ with i.i.d.\ regularly varying factors, the angular measure collapses to point masses at the normalized loading directions (Corollaries~\ref{cor:scalar}--\ref{cor}):
\[
\mathcal{H}_{A\bm{Z}}(\cdot) = \sum_{j=1}^{k} \lVert \bm{a}_j \rVert^{\alpha}\, \delta_{\bm{a}_j / \lVert \bm{a}_j \rVert}(\cdot).
\]
Its total mass is $\sum_j \lVert \bm{a}_j \rVert^{\alpha}$; after unit-Fr\'echet marginal standardization this matches the moment constraint of~\eqref{max_regvar}, so the TPDM below is well-defined.
    
    \begin{definition} (\cite{Cooley2019}) The TPDM corresponding to \(\bm{X} \in RV_{+}^n(\alpha)\) is defined as the \(n \times n\) matrix
    
    \[
    \Sigma_{\bm{X}} = \bigl( (\Sigma_{\bm{X}})_{iq} \bigr)_{i,q = 1, \dots, n} \text{, where } (\Sigma_{\bm{X}})_{iq} = \int_{\mathbb{S}_{+}^{n-1}} w_i w_q \, d\mathcal{H}_{\bm{X}}(\bm{w}) \text{.}
    \]
    
    \end{definition}
    
    By construction, $\Sigma_{\bm{X}}$ is symmetric ($(\Sigma_{\bm{X}})_{iq} = (\Sigma_{\bm{X}})_{qi}$) and entry-wise non-negative, since $w_i \ge 0$ on $\mathbb{S}_{+}^{n-1}$ and $\mathcal{H}_{\bm{X}}$ is a positive measure.  It is therefore a valid input for symmetric non-negative matrix factorization.
    The extremal dependence between components \(X_i\) and \(X_q\) is summarized by \((\Sigma_{\bm{X}})_{iq}\). Asymptotic independence of \(X_i\) and \(X_q\) is equivalent to \(\mathcal{H}_{\bm{X}} \left( \{ \bm{w} \in \mathbb{S}_{+}^{n-1} : w_i >0, w_q > 0 \} \right) = 0\), which corresponds to \((\Sigma_{\bm{X}})_{iq} = 0\) (\cite{Sibuya1960}, \cite{Schlather2003}).

    Under the conditions of Corollary~\ref{cor}, the TPDM of $A\bm{Z}$ has entries
    
    \[
    (\Sigma_{A\bm{Z}})_{iq} = \int_{\mathbb{S}_{+}^{n-1}} w_i w_q \, d\mathcal{H}_{A\bm{Z}}(\bm{w}) = \sum_{j=1}^{k} \left\Vert \bm{a}_j \right\Vert_2^{\alpha} \left( \frac{a_{ij}}{\left\Vert \bm{a}_j \right\Vert_2} \right) \left( \frac{a_{qj}}{\left\Vert \bm{a}_j \right\Vert_2} \right) \text{.}
    \]
    
    In financial applications, regularly varying distributions are commonly modeled with Pareto($\alpha$) marginals ($\mathbb{P}(X>x) = x^{-\alpha}$ for $x \ge 1$), where $\alpha = 2$ is a standard choice motivated by empirical evidence from stock prices and insurance claims (\cite{Hult2007,Cheng2024}) and by regulatory frameworks (\cite{BCBS2011}). Under $\alpha = 2$ the TPDM simplifies to a Gram matrix:
    \[
    (\Sigma_{A\bm{Z}})_{iq} = \sum_{j=1}^{k} a_{ij} a_{qj} \text{, i.e.\ } \Sigma_{A\bm{Z}} = A A^\top \text{.}
    \]
The SymNMF population target we use throughout is the \emph{unit-diagonal} normalization of this structural Gram matrix, $\Sigma = D^{-1/2}AA^\top D^{-1/2}$ with $D_{ii} = (AA^\top)_{ii}$ (Appendix~\ref{app:matrices}): $H$ therefore recovers the loadings of~$A$ only up to that per-row scaling. The empirical estimator (Section~\ref{estimation}) applies the same unit-diagonal normalization to the exceedance Gram $\tilde\Sigma_{\bm{X}}=(n/n_{\mathrm{exc}})\Omega^\top\Omega$, which estimates~$\Sigma$. Under the noiseless factor model of Corollary~\ref{cor} the population equality is exact; with same-index regularly varying noise as in our simulator, it is the leading factor contribution. We restrict to $\alpha = 2$ for the remainder of the paper.  For general $\alpha$, the (unnormalized) TPDM entry becomes $(\Sigma_{A\bm{Z}})_{iq} = \sum_j \|\bm{a}_j\|^{\alpha-2} a_{ij} a_{qj}$, a column-weighted Gram matrix that is still symmetric and non-negative but no longer equals $AA^\top$. SymNMF remains applicable, but the recovered factor $H$ no longer directly estimates the loading matrix $A$; interpreting $H$ then requires knowledge of $\alpha$. The algorithmic findings of this paper are properties of the SymNMF optimization problem itself and transfer to any $\alpha$; only the statistical interpretation of the factor changes.

With unit diagonal, $\sigma_{ii} = 1$ for every~$i$ and the trace equals~$n$:
    \[
    \operatorname{tr}(\Sigma_{\bm{X}}) = \sum_{i=1}^{n} \sigma_{ii} = n \text{.}
    \]
Since the correlation matrix also has unit diagonal and trace~$n$, both input types therefore present SymNMF with matrices of comparable Frobenius norm ($\|S\|_F^2 \ge \operatorname{tr}(S) = n$), so the normalized squared reconstruction error $E_t = \|S - HH^\top\|_F^2 / \|S\|_F^2$ is on a common scale across types.

\subsection{Empirical Estimation}\label{estimation}

Let \(\bm{x}_t\), \(t=1,\dots,n_{\mathrm{sample}}\), be i.i.d.\ observations from a regularly varying random vector with Pareto(2) marginals. In practice, marginals are transformed to this form.
    
Denote \(r_t = \left\Vert \bm{x}_t \right\Vert_2\) and \(\bm{w}_t = \bm{x}_t / \Vert \bm{x}_t \Vert_2\), for all \(t = 1,\dots,n_{\mathrm{sample}}\), using the polar coordinate transform of Section~\ref{regvar}. By construction, the TPDM depends only on the angular measure on the positive orthant of the unit sphere, \(\mathbb{S}_{+}^{n-1} = \{ \bm{w} \in \mathbb{R}_{+}^n : \left\Vert \bm{w} \right\Vert_2 = 1 \}\).

In practice, the theoretical angular measure is not known. Denote its empirical estimate as \(\hat{\mathcal{H}}_{\bm{X}}(\cdot)\). Consider the probability measure
    
    \[
    \hat{N}_{\bm{X}}(\cdot) = \frac{1}{n_{\mathrm{exc}}} \sum_{t=1}^{n_{\mathrm{sample}}} \delta_{\bm{w}_t}(\cdot)\, \mathbf{1}\{r_t > r_0\} \text{,}
    \] where \(r_0\) is a high threshold for radial coordinates and
    
    \[
    n_{\mathrm{exc}} = \sum_{t=1}^{n_{\mathrm{sample}}} \mathbf{1}\{r_t > r_0\}. 
    \]
    
Then, \(\hat{\mathcal{H}}_{\bm{X}}(\cdot) = n \hat{N}_{\bm{X}}(\cdot)\). By construction, the mass appears precisely at the angular components of the observations corresponding to the exceedances above the radial threshold \(r_0\). The unnormalized empirical Gram matrix is
    
    \[
    \tilde{\Sigma}_{\bm{X}} = \frac{n}{n_{\mathrm{exc}}} \Omega^\top \Omega,
    \]
    
where the rows of the $n_{\mathrm{exc}} \times n$ matrix $\Omega$ are the angular vectors $\bm{w}_t$ for all $t$ such that $r_t > r_0$.  The SymNMF input is the unit-diagonal normalization $\hat{\Sigma}_{\bm{X}} = D^{-1/2}\tilde{\Sigma}_{\bm{X}} D^{-1/2}$ with $D_{ii} = (\tilde{\Sigma}_{\bm{X}})_{ii}$, matching the population target of Section~\ref{tpdm} and Appendix~\ref{app:matrices}.  Since $\hat\Sigma_{\bm{X}}$ is a (normalized) Gram matrix of $n_{\mathrm{exc}}$ vectors, $\operatorname{rank}(\hat\Sigma_{\bm{X}}) \le n_{\mathrm{exc}}$; when $n_{\mathrm{exc}} < n$ the empirical TPDM is rank-deficient, which affects the SymNMF landscape (Section~\ref{snmf:spectral}).  The choice of threshold $r_0$ (equivalently, the exceedance fraction~$q$) controls a bias--variance trade-off: too low includes non-tail observations; too high leaves $n_{\mathrm{exc}}$ small and the estimate noisy.  In our benchmarks we use the top $1\%$ of radial norms ($q = 0.01$); details are in Appendix~\ref{app:matrices}.

\subsection{Low-Rank Decomposition}\label{decompositions}

For \(\bm{X} \in RV_{+}^n(2)\), \cite{Cooley2019} prove the existence of \(k^* < \infty\) such that $\Sigma_{\bm{X}} = A_{k^*} A_{k^*}^\top$ with $A_{k^*} \in \mathbb{R}_+^{n \times k^*}$, and propose two exact decompositions. We instead pursue an \emph{approximate} low-rank decomposition $\Sigma_{\bm{X}} \approx AA^\top$ with $A \in \mathbb{R}_+^{n \times k}$ and $k \ll n$. The same framework applies to both matrix types: for the TPDM, non-negativity of $A$ is structurally guaranteed; for correlation matrices, it holds when using absolute correlations.

\section{Experimental Methodology}\label{methodology}

This section develops the experimental infrastructure: a unified convergence criterion (Section~\ref{snmf:convergence}), the spectral properties that govern optimization difficulty (Section~\ref{snmf:spectral}), the GPU computational pipeline (Section~\ref{snmf:compute}), and the benchmark matrix construction (Section~\ref{snmf:matrices}).

\subsection{Convergence Criterion}\label{snmf:convergence}
A fair comparison of optimization algorithms requires a single, theoretically grounded stopping rule applied uniformly across all solvers.
In the NMF literature, three classes of stopping criteria are common:
(i)~relative decrease in the objective, $|E_{t} - E_{t-1}|/E_{t-1} < \tau_1$;
(ii)~relative change in the iterates, $\|H^{(t)} - H^{(t-1)}\|_F / \|H^{(t-1)}\|_F < \tau_2$;
(iii)~first-order optimality residuals (with their own tolerance $\tau_3$).
The first two are widely used \cite{lee2001algorithms,he2011symmetric} but have well-known shortcomings: (i)~can trigger far from stationarity when iterates traverse a nearly flat valley, and the denominator $E_{t-1}$ becomes numerically unstable in single precision as $E_t \to 0$; (ii)~can trigger from small step sizes alone even with a large gradient, and provides no quality guarantee.
Neither alone ensures proximity to a first-order stationary point.

For $\min_{H \ge 0} f(H)$, the first-order necessary conditions are the \emph{Karush--Kuhn--Tucker} (KKT) conditions \cite{bertsekas1999nonlinear}.
At a stationary point $H^*$ they require
\begin{equation}\label{eq:kkt}
[\nabla_H f(H^*)]_{ij} \;\ge\; 0 \;\;\text{for all } i,j, \qquad
H^*_{ij}\,[\nabla_H f(H^*)]_{ij} = 0 \;\;\text{for all } i,j, \qquad
H^*_{ij} \;\ge\; 0 \;\;\text{for all } i,j.
\end{equation}
Following \cite{lin2007projected}, these are summarized by the \emph{projected gradient}:
\begin{equation}\label{eq:proj_grad}
[\nabla_{\mathrm{proj}} f(H)]_{ij} \;=\; \begin{cases}
[\nabla_H f(H)]_{ij} & \text{if } H_{ij} > 0 \quad \text{(interior)}, \\
\min\!\bigl(0,\; [\nabla_H f(H)]_{ij}\bigr) & \text{if } H_{ij} = 0 \quad \text{(boundary)}.
\end{cases}
\end{equation}
On the boundary this keeps only components that point into the feasible region (negative entries of $\nabla_H f$); non-negative boundary entries already satisfy stationarity and contribute zero. The norm $\|\nabla_{\mathrm{proj}} f(H)\|_F$ vanishes if and only if~\eqref{eq:kkt} holds.

Our stopping rule requires three criteria simultaneously.

\medskip\noindent\textbf{Criterion 1: Loss gate.}
\begin{equation}\label{eq:loss_gate}
E_t \;<\; \tau_E, \qquad \tau_E = 0.1.
\end{equation}
The normalized squared reconstruction error $E_t = \|S - H^{(t)}{H^{(t)}}^\top\|_F^2 / \|S\|_F^2$ must be below~$\tau_E$, i.e.\ the residual carries at most $10\%$ of the squared Frobenius energy of~$S$ (a relative Frobenius residual of ${\sim}32\%$).
This is a deliberately loose anti-degeneracy guard against KKT points with meaningless reconstruction (notably $H = 0$, which satisfies $\nabla_H f = 0$ yet $E_t = 1$); the projected-gradient gate below is the binding stationarity criterion.

\medskip\noindent\textbf{Criterion 2: Per-element KKT.}
\begin{equation}\label{eq:kkt_stop}
\frac{\|\nabla_{\mathrm{proj}} f(H^{(t)})\|_F}{n \cdot k} \;<\; \tau_g, \qquad \tau_g = 10^{-4}.
\end{equation}
Dividing by $nk$, the number of entries in $H$, gives an entry-normalized stationarity measure, so a single threshold $\tau_g$ stays meaningful across problem sizes rather than tracking the growing raw norm; this follows the element-normalized convention of Kim \& Park~\cite{kim2014algorithms}.
The alternative normalization by the initial projected gradient~\cite{lin2007projected,gillis2020nmf} itself scales with $n$ and $k$, and in our experiments caused false convergence at $n = 10^4$ after ${\sim}10$ iterations.

When $k$ is below the number of significant eigenvalues of~$S$, the rank-$k$ factorization incurs irreducible approximation error whose gradient cannot be eliminated. At $n = 10^2$ ($k = 10$, true model rank $11$), no solver achieves $\|\nabla_{\mathrm{proj}}\|_F / (nk) < 10^{-4}$ even after $10^5$ iterations; increasing to $k = 25$ in the fixed-rank stress test (Section~\ref{snmf:k25}) restores convergence. To keep the benchmark informative at small scales, we use a scaled threshold:
\begin{equation}\label{eq:kkt_scaled}
\tau_g(n) \;=\; \tau_g^{\infty} \cdot \max\!\Bigl(1,\; \frac{n_{\mathrm{ref}}}{n}\Bigr), \qquad \tau_g^{\infty} = 10^{-4},\; n_{\mathrm{ref}} = 10^4,
\end{equation}
relaxing to $10^{-3}$ at $n = 10^3$ and $10^{-2}$ at $n = 10^2$. For $n \ge 10^4$ we use $\tau_g = 10^{-4}$, except TPDM at $n = 10^6$: there $k = n_{\mathrm{exc}} = 400$ (Section~\ref{snmf:spectral}) makes $E_t \to 0$ and the projected gradient precision-limited near $10^{-3}$, so we certify at $\tau_g = 10^{-3}$ (Section~\ref{results:phase2}). This is a numerical floor, not a weaker standard.

\medskip\noindent\textbf{Criterion 3: Objective stagnation.}
\begin{equation}\label{eq:obj_stag}
\frac{E_{t-\delta} - E_t}{E_0} \;<\; \tau, \qquad \tau = 10^{-5},
\end{equation}
where $E_0 = \|S - H^{(0)}{H^{(0)}}^\top\|_F^2 / \|S\|_F^2$ is the error at initialization and $\delta$ is the check interval (default 10~iterations).
Normalizing by $E_0$ rather than $E_{t-\delta}$ stays stable when $E_t$ is near zero in single precision and avoids the pathological denominator of criterion~(i) above.
To avoid false positives in warm-up, this criterion is suppressed for the first 50~iterations.

\begin{equation}\label{eq:conv_and}
\text{converged} \;\Longleftrightarrow\; \underbrace{E_t < \tau_E}_{\text{loss gate}} \;\wedge\; \underbrace{\frac{\|\nabla_{\mathrm{proj}}\|_F}{nk} < \tau_g}_{\text{KKT}} \;\wedge\; \underbrace{\frac{E_{t-\delta} - E_t}{E_0} < \tau}_{\text{stagnation}}.
\end{equation}
Each alone is insufficient; together they require meaningful reconstruction, approximate stationarity, and exhausted progress.

\paragraph{LogSpace adaptation.}
For solvers using $H = \operatorname{softplus}(\Theta) > 0$ (Section~\ref{snmf:logspace}), $H_{ij} > 0$ everywhere, so the projected gradient~\eqref{eq:proj_grad} equals $\nabla_H f$ and no boundary masking is needed.
We always evaluate the gradient in $H$-space (not $\Theta$-space): the factor $\sigma(\Theta)$ in~\eqref{eq:logspace_grad} would otherwise distort the norm and spoil cross-solver comparison.

At each check interval, every solver records iteration count, wall-clock time, $E_t$, $\|\nabla_{\mathrm{proj}}\|_F / (nk)$, and $(E_{t-\delta} - E_t)/E_0$.

\subsection{Spectral Properties of Input Matrices}\label{snmf:spectral}

With the convergence criterion in place, we now characterize the input matrices themselves: the spectral structure of~$S$ governs the optimization landscape and largely determines which algorithms succeed. Table~\ref{tab:spectral} reports key spectral properties of the two matrix types (TPDM and sample correlation), averaged over three seeds per configuration.
\begin{table}[t]
\centering
\caption{Spectral properties of TPDM and correlation benchmark matrices at $k = \lfloor\sqrt{n}\rfloor$ (mean over seeds 7, 42, 99). $\kappa^+ = \lambda_{\max}/\lambda_{\min}^+$ is the condition number over strictly positive eigenvalues (the signal-subspace conditioning for TPDM, the full-spectrum $\lambda_{\max}/\lambda_{\min}$ for correlation); $r_{\mathrm{eff}} = \operatorname{tr}(S)/\lambda_{\max}$ the effective rank; $\gamma_k = \lambda_k/\lambda_{k+1}$ and $\gamma_{k+1} = \lambda_{k+1}/\lambda_{k+2}$ the spectral-gap ratios at the factorization rank and one position below (the true model rank $k{+}1$); $\%\!\operatorname{var}_k$ the trace fraction in the top~$k$ eigenvalues. $\dagger$: full-spectrum $\kappa^+$ not meaningful (correlation is mildly indefinite at $n \ge 10^5$; see text). $\ddagger$: $k = 1000$ exceeds the fixed rank $n_{\mathrm{exc}} = 400$, so $\lambda_k = \lambda_{k+1} = 0$ and $\gamma_k, \gamma_{k+1}$ are undefined (see text).}
\label{tab:spectral}
\smallskip
\begin{tabular}{ll rrr rr rrr}
\toprule
Type & $n$ & $\lambda_1$ & $\lambda_k$ & $\lambda_{k+1}$ & $\kappa^+$ & $r_{\mathrm{eff}}$ & $\gamma_k$ & $\gamma_{k+1}$ & $\%\!\operatorname{var}_k$ \\
\midrule
\multirow{5}{*}{TPDM}
 & 100    & 54.0    & 2.82  & 0.57   & $\sim\!6 \times 10^{5}$  & 1.9  & 5.0 & 1.8 & 96.5 \\
 & 1\,000  & 528    & 5.18  & 2.91   & $\sim\!5 \times 10^{7}$ & 1.9  & 1.8 & 3.1 & 98.7 \\
 & 10\,000 & 5\,833  & 1.07  & 0.98   & $\sim\!2 \times 10^{7}$ & 1.7  & 1.1 & 1.0 & 99.8 \\
 & 100\,000 & 74\,935 & 0.33 & 0.33 & $\sim\!10^{7}$ & 1.3 & 1.0 & 1.0 & 99.98 \\
 & 1\,000\,000 & $8.7 \times 10^{5}$ & $0^{\ddagger}$ & $0^{\ddagger}$ & $\sim\!3 \times 10^{6}$ & 1.1 & n/a$^{\ddagger}$ & n/a$^{\ddagger}$ & 100.0 \\
\midrule
\multirow{5}{*}{CORR}
 & 100    & 27.7   & 3.09  & 0.74   & 137                & 3.6  & 4.2 & 1.4 & 69.3 \\
 & 1\,000  & 244    & 9.00  & 2.43   & 1\,480              & 4.1  & 3.7 & 3.7 & 67.8 \\
 & 10\,000 & 2\,362  & 27.2  & 14.7   & 33\,411             & 4.2  & 1.9 & 17.3 & 68.2 \\
 & 100\,000 & 23\,902 & 95.1 & 94.1 & n/a$^{\dagger}$ & 4.2 & 1.0 & 46.7 & 70.9 \\
 & 1\,000\,000 & $2.5 \times 10^{5}$ & 333 & 330 & n/a$^{\dagger}$ & 4.0 & 1.0 & 39.3 & 77.3 \\
\bottomrule
\end{tabular}
\end{table}
\paragraph{TPDM.}
The empirical TPDM (Section~\ref{estimation}; unit-diagonal normalized Gram matrix of exceedance angles) is positive semidefinite with $\operatorname{rank}(\hat\Sigma)\le n_{\mathrm{exc}}$: it is formed from $n_{\mathrm{exc}}=400$ tail vectors (top $q=1\%$ of $n_{\mathrm{sample}}=40{,}000$), so at most $400$ directions, whereas correlation uses the full sample and is full-rank. With $n_{\mathrm{exc}}$ fixed and $k=\lfloor\sqrt{n}\rfloor$ growing, the two cross near $n\approx1.6\times10^5$; by $n=10^6$ one has $k=1000>n_{\mathrm{exc}}$, so the factorization is over-parameterized and the top-$k$ subspace captures all signal (Table~\ref{tab:spectral}: $\lambda_k=\lambda_{k+1}=0$, $\%\!\operatorname{var}_k=100$). Population TPDM rank is $\le k^*$ under a $k^*$-factor model, but this empirical ceiling dominates and is an artifact of the fixed-exceedance generator, not of empirical TPDMs whose rank grows with the sample.

Spectrally, top-$k$ eigenvalues hold $>96\%$ of trace at all sizes, yet gaps collapse: by $n=10^4$, $\gamma_k$ falls $5.0\to1.1$ and $\gamma_{k+1}\to1.0$, bunching the leading spectrum and creating near-degenerate optima. Fixed $n_{\mathrm{exc}}$ against growing $n$ thus makes TPDM a harder SymNMF target at scale. Full-matrix $\kappa^+$ is inflated by the near-zero tail ($>\!10^{10}$ for $n\ge10^3$); the signal-subspace conditioning is ${\sim}10^7$ (Table~\ref{tab:spectral}). At $n=10^5$, $\gamma_k\approx1$ and $r_{\mathrm{eff}}\sim1.3$ as the common factor dominates ($\lambda_1\sim O(n)$).

\paragraph{Correlation.}
Built from the full sample, correlation stays full-rank after Ledoit--Wolf shrinkage when $n/n_{\mathrm{sample}}>0.1$ (Appendix~\ref{app:matrices}). The noise ridge $\sigma_\varepsilon^2 I$ keeps $\kappa^+$ moderate (${\sim}33{,}000$ at $n=10^4$) but lifts $r_{\mathrm{eff}}$ to ${\sim}4$, so top-$k$ captures only ${\sim}68\%$ of trace. Spectral separation strengthens with $n$ ($\gamma_{k+1}$: $1.4\to17.3\to46.7\to{\approx}39$; Table~\ref{tab:spectral}), consistent with bulk-edge sharpening \cite{ledoit2004wellconditioned}. For $n\ge10^5$, $\gamma_k\approx1$ because $\lambda_k$ and $\lambda_{k+1}$ are both signal (model rank $k{+}1$); the relevant gap is $\gamma_{k+1}$. Effective rank stays ${\approx}4$ while $\%\!\operatorname{var}_k$ rises only slowly ($68\%\to77\%$ by $n=10^6$), since growing $k=\lfloor\sqrt{n}\rfloor$ retains progressively more of the noise bulk.

Correlation is positive definite through $n=10^4$ ($\lambda_{\min}=0.07$). At $n=10^5$ a small negative eigenvalue appears ($\lambda_{\min}\approx-0.018$): absolute correlation need not stay positive definite, and the noise ridge no longer masks it. Relative indefiniteness is tiny ($|\lambda_{\min}|/\lambda_{\max}\sim7\times10^{-7}$), so $\kappa^+$ is undefined but $\gamma_{k+1}=46.7$ remains sharp. At $n=10^6$, $\lambda_{\min}\approx-0.17$ with the same relative scale, confirmed by distributed matrix-free Lanczos on the $4$\,TB matrix.

The two types are complementary: TPDM concentrates variance but loses gaps at scale; correlation spreads variance but keeps a clean signal--noise boundary. Neither is uniformly easier.

\subsection{Computational Methodology}\label{snmf:compute}

Every solver reduces, at each iteration, to a small number of GPU primitives, dominated by two dense cuBLAS products: $S \times H$ ($O(n^2 k)$, \texttt{sgemm}) and $H^\top H$ ($O(nk^2)$, \texttt{syrk}).\footnote{When $S$ is pre-sparsified (e.g.\ by thresholding), PyTorch dispatches $S \times H$ to cuSPARSE automatically. However, factor models with a common factor produce dense matrices, so all benchmarks here use dense cuBLAS.} The remaining work, the element-wise projection $\max(H, 0)$ and the optimizer state updates, is $O(nk)$ and negligible by comparison. For $k \ll n$, the $S \times H$ product dominates even $H^\top H$: at $n = 10^5$ a single dense \texttt{sgemm} runs at near-peak throughput (${\sim}3$\,ms), and non-GEMM overhead (projections, state updates, convergence checks) accounts for $<10\%$ of wall-clock at $n \ge 10^4$. Cross-solver wall-clock therefore largely tracks iteration count.

All solvers are implemented in PyTorch, with gradients computed analytically (Section~\ref{snmf:loss}) via cuBLAS and multi-GPU communication via \texttt{torch.distributed} (NCCL). Phase~1 ($n \le 10^4$) runs on a single NVIDIA GB200 (Grace--Blackwell) GPU. The trace-identity reformulation keeps the $n = 10^5$ matrix within single-GPU memory (Table~\ref{tab:memory}), but Phase~2 shards for throughput and capacity: $n = 10^5$ across the four GPUs of one node, and $n = 10^6$ row-sharded across 16 nodes (64 GPUs) once the matrix exceeds single-GPU memory (Section~\ref{results:phase2}). Reported timings and residuals are in fp32. As a precision check we also ran the same solves with TF32 Tensor-Core multiplies (10-bit mantissa, fp32 accumulation); for SymNMF the dominant rounding error arises from cancellation in the trace-identity loss (Section~\ref{snmf:loss}), not from individual matrix products, and TF32 leaves converged $E_t$ unchanged at the levels reported here ($E_t$ between $10^{-3}$ and $10^{-2}$). We benchmark at $n \in \{10^2,\, 10^3,\, 10^4,\, 10^5,\, 10^6\}$; first-order gradient, adaptive first-order, and second-order methods are run in both projected and LogSpace variants to test the memory-length prediction of Section~\ref{snmf:logspace}.

\subsection{Benchmark Matrix Construction}\label{snmf:matrices}

Both matrix types are generated from a shared factor model $X = AZ + \varepsilon$, where $A \in \mathbb{R}_+^{n \times (k+1)}$ is a block-structured mixing matrix with one common factor and $k = \lfloor\sqrt{n}\rfloor$ group factors of variable size (${\pm}\,40\%$ of $n/k$). Distributional assumptions match each matrix's foundation; sample size is fixed at $n_{\mathrm{sample}} = 40{,}000$ (Appendix~\ref{app:matrices}):

\begin{center}
\begin{tabular}{lll}
\toprule
 & \textbf{TPDM} & \textbf{Correlation} \\
\midrule
Factor distribution $Z$ & Pareto(2) & Gaussian \\
Idiosyncratic noise $\varepsilon$ & Pareto(2) & Gaussian \\
Estimator & Angular measure (top $1\%$ exceedances) & Absolute Pearson ($|C_{ij}|$) \\
\bottomrule
\end{tabular}
\end{center}

Both matrices are symmetric, non-negative, with unit diagonal and values in $[0, 1]$. Sharing $A$ gives the same latent factor structure: correlation captures full-sample dependence; the TPDM captures dependence conditional on extremes.

Our primary factorization rank is $k = \lfloor\sqrt{n}\rfloor$, used throughout unless stated otherwise
(e.g.\ $k=10$ at $n=10^2$, up to $k=1000$ at $n=10^6$).
The true generative rank is one higher ($\lfloor\sqrt{n}\rfloor + 1$: common factor plus $k$ group factors); setting the factorization rank one below forces solvers to absorb the common factor into the group columns of $H$, a mild misspecification that practitioners face by default. The $\sqrt{n}$ scaling balances expressiveness and parsimony ($\sqrt{n}$ groups of ${\sim}\!\sqrt{n}$ members) and arises in random matrix theory \cite{marchenko1967distribution} and community detection \cite{lei2015consistency}. At Phase~1 scales ($n \le 10^4$) we additionally stress-test a \emph{fixed} rank $k_{\mathrm{fix}} = 25$, calibrated to the ${\sim}\!15$--$70$ factors in commercial equity models such as Barra USE4 \cite{menchero2011barra}. That under-ranked setting checks robustness when $k$ is far below the generative rank (Section~\ref{snmf:k25}); we do not repeat the fixed-$k$ stress test in Phase~2. Rows of $A$ are randomly permuted before computing $S$ so solvers cannot exploit block-diagonal structure. Full details are in Appendix~\ref{app:matrices}.

\paragraph{Initialization and seeds.}\label{snmf:init}
All solvers share the same scaled random non-negative initialization per seed (entry-wise $\sim \mathcal{U}[0, \sqrt{\bar{S}/k}]$, with $\bar{S}$ the mean entry of~$S$); one MU variant additionally uses NNDSVD \cite{boutsidis2008svd}. We run three seeds per configuration and report best and mean $E_t$. Solution quality is highly reproducible (coefficient of variation of $E_t$ below $2\%$ on correlation and below $15\%$ on TPDM at $n = 10^4$); the elevated TPDM variance reflects near-singular spectra (Table~\ref{tab:spectral}), not solver instability. Large-scale runs ($n \ge 10^5$) report each solver's best configuration confirmed across three seeds.

\section{Results}\label{results}

Adaptive first-order methods (AdaGrad and its stochastic variants) alone combine universal convergence with scaling to $n = 10^6$; classical baselines (multiplicative updates, projected gradient, second-order, and deep unfolding) stall on hard TPDM spectra or trail on wall-clock. We establish this with a two-phase funnel. Phase~1 (Section~\ref{snmf:perf}) screens all seven families at $n \le 10^4$ by the strict convergence criterion and then cost. Phase~2 (Section~\ref{results:phase2}) tests those that meet both thresholds on $n \in \{10^5,\, 10^6\}$, where a failed run can cost tens of GPU-hours. Between the two, Section~\ref{results:logspace} isolates projection versus softplus and shows that the effect is mediated by optimizer memory length.

\subsection{Cross-Family Comparison}\label{snmf:perf}

Phase~1 spans $n \in \{10^2,\, 10^3,\, 10^4\}$, with $n = 10^4$ decisive for selection. We report $E_t$, wall-clock, and the projected-gradient norm under~\eqref{eq:conv_and} (Section~\ref{snmf:convergence}). Ranking by $E_t$ alone would promote line-search and mirror-descent variants that miss the KKT gate.

\subsubsection*{Phase 1: Algorithm Selection ($n \le 10^4$)}

We evaluate over 100 hyperparameter runs across the 30+ configurations of Section~\ref{snmf:algos}, covering all seven families at $n \in \{10^2,\, 10^3,\, 10^4\}$ on both TPDM and correlation matrices at $k = \lfloor\sqrt{n}\rfloor$.
Each configuration uses 3~seeds (6~runs per size, 18~total). A fixed-rank $k = 25$ stress test is reported separately (Section~\ref{snmf:k25}).
Table~\ref{tab:phase1} reports the best configuration per method, ranked by convergence rate then wall time at $n = 10^4$.

\begin{table}[ht]
\centering
\caption{Phase~1: best configuration per tabulated method on 18~runs ($n \in \{10^2,10^3,10^4\}$, both types, $k=\lfloor\sqrt{n}\rfloor$, 3~seeds; GB200). Within-family variants that converge on no $n = 10^4$ run (PGD, heavy-ball, exponentiated gradient, trace-norm MU, ANLS, L-BFGS, PGNCG) are omitted here and discussed in Section~\ref{snmf:perf}.
Conv.\ = runs satisfying~\eqref{eq:conv_and}; wall columns = mean seconds among converged runs (6~per size); ``---'' = none converged; Params = tunable hyperparameters excluding the shared iteration budget.
Ordered by convergence then wall at $n=10^4$. Dashed rule: 100\% and ${<}\,4$\,s at $n=10^4$ (used in Phase~2); solid rule: remaining 100\% methods (above 4\,s).}\label{tab:phase1}
\small
\begin{tabular}{llccccr}
\toprule
\textbf{Method} & \textbf{Family} & \textbf{Conv.} & \multicolumn{3}{c}{\textbf{Wall (s) at $n$}} & \textbf{Params} \\
\cmidrule(lr){4-6}
 & & & \textbf{100} & \textbf{1{,}000} & \textbf{10{,}000} & \\
\midrule
ADMM                    & F6 Splitting     & 18/18 & 0.35 & 0.22 & 0.48 & 1 \\
RMSprop                 & F3 Adaptive      & 18/18 & 0.35 & 0.44 & 0.99 & 3 \\
Block-SVRG \textsc{AdaptGrow} & F3 Adaptive & 18/18 & 0.56 & 0.72 & 1.67 & 3 \\
AdaGrad                 & F3 Adaptive      & 18/18 & 1.06 & 0.98 & 1.68 & 1 \\
Row-Stoch.\ SVRG        & F3 Adaptive      & 18/18 & 0.68 & 1.65 & 2.65 & 3 \\
Piecewise AdaGrad       & F3 Adaptive      & 18/18 & 2.19 & 0.66 & 3.22 & 2 \\
\hdashline
Adan                    & F3 Adaptive      & 18/18 & 1.33 & 1.34 & 4.64 & 2 \\
HALS                    & F4 BCD           & 18/18 & 0.78 & 1.42 & 8.66 & 1 \\
LAI-SymNMF              & F6 Splitting     & 18/18 & 2.15 & 4.11 & 15.22 & 2 \\
Adam                    & F3 Adaptive      & 18/18 & 0.70 & 1.95 & 17.03 & 2 \\
NAdam                   & F3 Adaptive      & 18/18 & 0.70 & 2.17 & 17.64 & 2 \\
\midrule
MU              & F1 Baseline      & 15/18 & 0.17 & 0.62 & 7.89  & 0 \\
APG / FISTA     & F2 PGD           & 15/18 & 0.12 & 0.38 & 16.49 & 1 \\
Newton (diag.)  & F5 Second-order  & 15/18 & 0.37 & 3.45 & 24.61 & 2 \\
SymNMF-Net      & F7 Deep Unfold.  & 12/18 & 36.08 & 44.82 & --- & 3+ \\
\bottomrule
\end{tabular}
\end{table}

At $n = 10^2$ and $n = 10^3$, most families converge reliably (even MU and PGD reach low $E_t$ within budget), and $n = 10^3$ already exceeds typical prior SymNMF sizes \cite{kuang2015symnmf,he2011symmetric}. Differentiation starts at $n = 10^4$, where spectra sharpen (Section~\ref{snmf:spectral}) and weaker methods fail.

Eleven methods converge on all 18~runs: ADMM, RMSprop, AdaGrad, Piecewise AdaGrad, Row-Stochastic SVRG, Block-SVRG \textsc{AdaptGrow}, Adan, NAdam, Adam, HALS, and randomized LAI-SymNMF~\cite{hayashi2024randomized}.
Convergence alone no longer separates them; \emph{cost} does.
Six finish under 4\,s at $n = 10^4$ (dashed rule in Table~\ref{tab:phase1}); the other five (Adan, HALS, LAI-SymNMF, Adam, NAdam) are $1.4$--$5\times$ slower.
The gap is iteration count, not per-step cost: dense solvers share a ${\sim}\,1.3$--$1.7$\,ms floor (Section~\ref{snmf:compute}), so Adam's $6{,}900$--$8{,}300$ iterations cost far more wall time than AdaGrad's $450$--$1{,}080$.
We discuss each family next, focusing on $n = 10^4$.

\paragraph{Family 1: Multiplicative / Mirror Descent.}
At $n = 10^4$ MU fails on all three TPDM runs: the projected gradient stalls ${\sim}\,10\times$ above threshold through the full iteration budget, matching the sublinear rate of~\eqref{eq:mu_update}. Trace-norm MU is highly $\lambda$-sensitive (fails once $\lambda \ge 1$), and mirror descent (exponentiated gradient) fails on every $n = 10^4$ run.

\paragraph{Family 2: Projected Gradient Methods.}
APG with Armijo backtracking and adaptive restart fails only on $n = 10^4$ TPDM, where FISTA momentum~\cite{beck2009fast} overshoots on the near-unit gap ($\gamma_k \approx 1$; Table~\ref{tab:spectral}), restarts, and loses acceleration. Plain PGD and heavy-ball fail on every $n = 10^4$ run. LogSpace variants do not improve (Section~\ref{snmf:logspace}).

\paragraph{Family 3: Adaptive First-Order Methods.}
RMSprop with $(\eta, \mu, \alpha) = (0.05,\, 0.5,\, 0.99)$ uses a short EMA window ($\alpha = 0.99$, ${\sim}100$ iterations, vs.\ Adam's ${\sim}1{,}000$) for per-coordinate scaling, moderate momentum ($\mu = 0.5$) for acceleration without projection conflicts, and an aggressive learning rate ($\eta = 0.05$).
It needs 3~hyperparameters and is sensitive to each: $\alpha = 0.9$ drops convergence by 30--40\%, $\mu = 0.9$ causes oscillation, and $\eta$ must be tuned per scale.

AdaGrad with $\eta = 2.0$ matches that reliability with a single hyperparameter.
Adam ($\beta_1 = 0.9$, $\beta_2 = 0.999$) converges too, but needs far more iterations at the shared dense-step floor (Table~\ref{tab:phase1}).
That gap matches the stationary SymNMF landscape: AdaGrad's infinite accumulator $G_t = \sum_{i=1}^{t} g_i^2$ keeps early gradient statistics informative (no distribution shift), whereas Adam's forgetting ($\beta_2 = 0.999$) discards them and its first-moment EMA interacts poorly with clamp projection (Section~\ref{snmf:logspace}).
NAdam and Adan behave like Adam across 80+~Family-3 configurations, and all three degrade under $k = 25$ (Section~\ref{snmf:k25}).

This creates a tension. AdaGrad's infinite accumulator is the \emph{right} memory policy for the stationary landscape, aggregating curvature across iterations, but the \emph{wrong} policy for long runs: growing $G_t$ exhausts the effective learning-rate budget, and the high optimal $\eta = 2.0$ only delays saturation to ${\sim}\,600$ iterations.
RMSprop escapes saturation by exponential forgetting, but then discards the stationarity advantage that makes AdaGrad strong.
To keep full-history scaling while escaping saturation, we propose \textbf{Piecewise AdaGrad} (likewise the stochastic AdaGrad variants below; no regret or KKT-rate claims): run pure AdaGrad until objective stagnation (criterion~3 of Section~\ref{snmf:convergence}), recalibrate the accumulator per entry to the current gradient scale, and resume.
Reset mechanics are adaptive, with one exposed scale $c$ (\texttt{memory\_iters}) and a large-scale gradient-margin multiplier $c_r$:
\emph{when} to reset follows the existing stagnation criterion;
\emph{whether} to reset again depends on whether the previous reset reduced the loss (self-terminating);
\emph{how long} to wait after a reset follows accumulator doubling;
\emph{how much} uses a median-gradient baseline
$G_i \leftarrow \max(c \cdot g_{t,i}^2,\; c \cdot \mathrm{median}(g_t^2))$,
with $c$ the only new parameter.
Between resets the method keeps AdaGrad's full history; the reset is a phase-boundary detector, not RMSprop-style continuous EMA forgetting.

With $(\eta, c) = (1.0, 10)$, the reset should fire only when the iterate is stale, not merely far from the optimum. An absolute cap ($E_t < 0.5$) works at small scale but fails at $n = 10^6$, where random initialization keeps $E_t$ above any fixed cap through a long saturated crawl (reset suppressed for ${\sim}1{,}900$ iterations after the objective has stalled by ${\sim}150$); removing the cap over-resets near convergence at $n = 10^5$.
We therefore gate on a \emph{scale-free} stationarity margin ($\lVert \nabla_{\mathrm{proj}}\rVert_F/(nk) > c_r\,\tau_g$, $c_r \in \{3, 10, 30\}$), which escapes the large-scale plateau early and self-terminates near the KKT gate; a median-squared-gradient floor additionally blocks premature early-phase resets.
Where diagonal AdaGrad converges quickly, no reset fires and the two methods agree.
The reset cost appears only at $n = 10^2$, where AdaGrad never saturates: the detector sometimes trips on a shallow early plateau and wastes iterations, even though per-iteration cost is \emph{lower} than diagonal AdaGrad ($1.22$ vs.\ $1.48$\,ms).
Aggressive $c = 1$ diverges on large-$k$ correlation, so the reset scale matters.
Versus RMSprop, Piecewise has fewer hyperparameters but is slower at Phase~1 (Table~\ref{tab:phase1}).

\textbf{Row-Stochastic AdaGrad} samples a fraction of rows per iteration, cutting per-step cost from $O(n^2 k)$ to $O(|I|\,n\,k)$.
Full-batch ($|I| = n$, $\eta = 2.0$) matches diagonal AdaGrad, validating the implementation.
Naive sub-sampling ($|I|/n \le 0.5$) fails the KKT gate on several $n = 10^4$ runs.
Two mechanisms drive this: (a)~row-sampling gradient noise, and (b)~uneven accumulator growth on infrequently sampled rows.

We address (b) with \textbf{Row-Stochastic SVRG}: periodic full-gradient snapshots $\mu$, fresh gradients on sampled rows and snapshot gradients elsewhere, so every row of $H$ updates each iteration and the accumulator grows uniformly while retaining $O(|I|\,n\,k)$ cost between snapshots.
With $(\eta, |I|/n) = (2.0, 0.5)$ that closes the gap naive sub-sampling could not.
A \emph{stochastic reset} hybrid (row sampling + Piecewise resets) fails: noise plus resets is destabilizing, and the reset detector, tuned to smooth full-batch loss, misreads stochastic stagnation.
At Phase~1 the dense step already sits on the ${\sim}\,1.3$--$1.7$\,ms floor, so row sub-sampling buys no wall-clock here.
\paragraph{Entry-stochastic sampling for AdaGrad.}
The row-level stochastic variants above sample a subset $I \subset [n]$ of rows and compute the \emph{exact} partial gradient for those rows.
A natural alternative, suggested by the structure of the problem, is to sample individual \emph{entries} of~$S$ \cite{duchi2011adagrad}: pick $m$ pairs $(i,j)$ uniformly from $[n] \times [n]$, compute the per-entry gradient contributions
\begin{equation}\label{eq:entry_grad}
  g_i \mathrel{+}= 2\bigl(h_i^\top h_j - S_{ij}\bigr)\,h_j, \qquad
  g_j \mathrel{+}= 2\bigl(h_i^\top h_j - S_{ij}\bigr)\,h_i,
\end{equation}
and update only the affected rows of~$H$.
Each sampled entry costs $O(k)$, so the per-iteration cost is $O(mk)$, fully sub-quadratic when $m \ll n^2$, versus $O(|I|\,nk)$ for row sampling, which remains $O(n^2 k)$ unless $|I| \ll n$.

Entry sampling is the natural stochastic gradient for the bilinear SymNMF loss: since $f(H) = \sum_{i,j}(S_{ij} - h_i^\top h_j)^2$, each $(i,j)$ pair provides an unbiased gradient estimator (up to a global factor absorbed into~$\eta$).
We implement entry-stochastic variants of all three row-sampled methods: base stochastic AdaGrad, stochastic SVRG, and stochastic reset.

However, entry sampling with a fixed small fraction ($m/n^2 \le 0.2$) fails to converge at $n = 10^3$ and above, across all three variants and twelve hyperparameter settings per size.
The objective value $E_t$ decreases to roughly the correct magnitude ($10^{-2}$--$10^{-3}$), but the projected-gradient KKT norm $\|\nabla_{\mathrm{proj}} f(H)\|_F / (nk)$ remains 100--200$\times$ above the convergence threshold.

The failure is a \emph{criterion mismatch}, not an information-theoretic impossibility of SGD: updates use sparse entry gradients, while certification requires the exact full gradient $\nabla_H f(H) = 4\bigl(H(H^\top H) - SH\bigr)$ below $\tau_g = 10^{-4}$.
With entry fraction $m/n^2 = 0.01$ at $n = 10^4$, each iteration sees only $10^6$ of $10^8$ entries, so the certified full-gradient residual remains large long after the stochastic objective has plateaued.
Closing that residual at fixed fraction would require an impractically large iteration budget under our strict gate.

\paragraph{Variance reduction at fixed entry fraction does not close the KKT gap.}
We tested three orthogonal variance-reduction strategies (SVRG \cite{johnson2013svrg}, Polyak--Ruppert iterate averaging \cite{polyak1992acceleration, ruppert1988efficient}, and dual averaging via the ADAGRAD-RDA framework of \cite{xiao2010dual}), each closing only ${\sim}\,6\%$ of the KKT gap at fixed entry fraction.  All three modify how stochastic gradient information is \emph{aggregated} rather than how much per-iteration entry budget is available, so the certification mismatch above is unaffected.  Detailed treatment, including the dual-averaging proximal step and its primal warm-up, is in Appendix~\ref{app:failed}.

\paragraph{Sketching is not a natural alternative.}
Replacing $S$ with a low-rank sketch \cite{halko2011finding, woodruff2014sketching} introduces a bias $\|S - \hat S\|_F$ that, for the dependence matrices in this work (slowly decaying spectra; Table~\ref{tab:spectral}), exceeds the KKT threshold unless the sketch rank is $r = O(\sqrt{n})$ or larger, at which point the sub-quadratic advantage erodes.  Practitioners also typically need the factorization of the \emph{true} $S$ rather than a surrogate.  Detailed argument in Appendix~\ref{app:failed}.

\paragraph{Adaptive growing-batch stochastic AdaGrad.}
The analysis above identifies a fundamental tension: fixed small entry fractions provide cheap per-step cost but cannot satisfy strict KKT, while full-batch iterations are expensive but guarantee convergence.
The resolution is to \emph{start} with a small entry fraction for fast initial progress and \emph{grow} it adaptively until the gradient is exact enough for KKT certification.

This strategy is theoretically grounded in the ``norm test'' of \cite{friedlander2012hybrid}: increase the sample size when the stochastic gradient norm exceeds a fraction of the true gradient norm, i.e., when gradient noise dominates algorithmic progress.
\cite{byrd2012sample} provide complementary sample complexity bounds, showing that the optimal sample size grows as convergence is approached.
The key insight is that stagnation is precisely the signal that gradient noise is the bottleneck, not optimization dynamics.

We implement an \textbf{Adaptive Growing-Batch} (\textsc{AdaptGrow}) variant that uses stagnation detection from the existing convergence criterion (Section~\ref{snmf:convergence}) as the growth trigger:
\begin{enumerate}[leftmargin=2em,topsep=2pt,itemsep=1pt]
  \item Initialise with entry fraction $\phi_0$ (e.g., $\phi_0 = 0.1$), so each iteration samples $m = \phi_0 n^2$ entries at cost $O(\phi_0 n^2 k)$.
  \item After each convergence check (every 10~iterations), if the objective stagnation criterion fires ($|\Delta E_t| < \texttt{tol}$) for 3~consecutive checks, \emph{double} the entry fraction: $\phi \leftarrow \min(2\phi,\, 1)$.
  \item Once $\phi = 1$, every iteration computes the exact full gradient; subsequent steps are diagonal AdaGrad, and the KKT norm converges to zero at the standard rate.
\end{enumerate}
Crucially, no schedule hyperparameters are introduced: the doubling rate and timing are entirely determined by the solver's own stagnation diagnostics, following the norm-test principle of \cite{friedlander2012hybrid}.
The only exposed parameters are the learning rate $\eta$ and the initial entry fraction $\phi_0$.

The growth is geometrically self-regulating: cheap noisy steps drive early progress and growth fires only as the iterate nears stationarity, so the batch reaches full size after at most $\lceil\log_2(1/\phi_0)\rceil$ doublings and the few expensive full-batch iterations occur when the iterate is already near-converged ($T_{\mathrm{final}} \ll T_{\max}$).

We implement \textsc{AdaptGrow} for all three entry-stochastic variants (base, SVRG, reset).
At $n = 10^3$ with $(\eta, \phi_0) = (1.0, 0.25)$, the uniform-entry variants are competitive with the Phase~1 baselines; the entry-sampled SVRG snapshot fires only a handful of times before the batch saturates.

The picture changes at $n = 10^4$, but in \emph{cost}, not convergence.
Uniform-entry variants still meet the Phase~1 gates, yet each iteration is dominated by sparse scatter--gather: cuSPARSE is an order of magnitude slower than a dense step of the same FLOP count, so wall time lags full-batch AdaGrad despite far fewer iterations.
The bottleneck is the sampling implementation, not the optimizer: this motivates the block-structured estimator below, which restores dense cuBLAS matmuls but introduces an accumulator pathology that an SVRG hybrid then repairs.

\paragraph{Block-structured sampling and the AdaGrad accumulator.}
To restore dense matmuls, we implement a \emph{block-structured} entry-stochastic estimator: at each iteration sample row and column index sets $I, J \subset [n]$ with $|I| = |J| = \lceil\sqrt{\phi}\,n\rceil$ and compute the gradient from the $I \times J$ residual sub-block via three dense (cuBLAS) matmuls.
The total entry coverage $|I|\cdot|J| = \phi n^2$ matches naive uniform entry sampling; without explicit reweighting of missing pairs the block residual is not an unbiased estimator of the full gradient, but AdaGrad's per-element scaling absorbs a global constant. The more serious issue is per-row imbalance in the accumulator, as follows.

Each block step updates only rows of $H$ indexed by $I \cup J$. A touched row aggregates $|I|$ or $|J|$ pair residuals at once, so its squared gradient is roughly $|J|^2$ times a single-pair term (since $|I|=|J|$), while untouched rows get $g = 0$ and leave $G_i$ frozen.
Over the stochastic phase, frequently sampled rows therefore build a much larger $G_i = \sum_s g_{s,i}^2$ than rarely sampled ones.
When \textsc{AdaptGrow} saturates ($\phi = 1$) and the gradient becomes dense and exact, steps still use $1/\sqrt{G_i}$ from that uneven history: large $G_i$ yields vanishing steps that cannot clear the KKT residual; small $G_i$ yields oversized steps that oscillate.
On a well-specified rank this stays latent through Phase~1, but under $k = 25$ (Section~\ref{snmf:k25}) every $n = 10^4$ run fails to converge ($E_t \sim 10^{-2}$, projected gradient ${\sim}6\times$ above threshold). The imbalance grows with the length of the stochastic phase, so the same failure is expected at larger $n$ even when $k = \lfloor\sqrt{n}\rfloor$.

The fix is an \emph{SVRG hybrid snapshot} in the stochastic phase. At a frequency of $\sim 1/\phi$ iterations we compute $\mu = \nabla f(H)$ and use $\mu_i$ as a placeholder gradient on rows outside $I \cup J$ (replaced by the fresh block estimate on touched rows), so untouched rows still contribute $\mu_i^2$ to~$G$ and the per-row update density is smoothed \emph{preemptively}, before heterogeneity accumulates.
With $\phi_0 = 0.5$ (each block already covers ${\approx}\,70\%$ of rows per iteration), the SVRG hybrid runs at dense ${\sim}\,1.35$\,ms/iteration, an order of magnitude faster than uniform entry sampling, and passes the $k = 25$ stress test where plain block sampling fails (Section~\ref{snmf:k25}). A Piecewise-style reset once the batch is full does \emph{not} repair heterogeneity and hurts when stacked on the snapshot (Appendix~\ref{app:failed}).

\paragraph{Why this construction is specific to entry-stochastic AdaGrad.}
Adaptive batch growth adds nothing to \emph{row-sampled} stochastic AdaGrad (already reliable via SVRG, and its value at scale is precisely in \emph{not} growing back to full batch), and it cannot be combined with RMSprop under the natural dense masking of unsampled coordinates: the EMA decays on zeros while AdaGrad's $G$ freezes ($g_t=0$ leaves $G_t$ unchanged). All four stochastic variants in this work (row-sampled, entry-sampled, SVRG, \textsc{AdaptGrow}) are therefore built on AdaGrad; the structural argument, including why SVRG does not transfer to RMSprop, is in Appendix~\ref{app:failed}.

\paragraph{Family 4: Block Coordinate Descent.}
HALS ($\alpha = 10^2$) converges reliably but is far slower at $n = 10^4$ (Table~\ref{tab:phase1}): column-by-column sweeps inflate both iteration count and per-step cost. ANLS fails on every $n = 10^4$ run.

\paragraph{Family 5: Second-Order Methods.}
Diagonal Newton fails on all $n = 10^4$ TPDM runs (diagonal curvature misses $HH^\top$ coupling) and, even on correlation, stays far slower than RMSprop at the same $O(n^2 k)$ per-step cost. L-BFGS converges on no projected run (LogSpace works only at small scale; Section~\ref{results:logspace}). PGNCG fails every $n = 10^4$ run under our criterion.

\paragraph{Family 6: Splitting Methods.}
ADMM ($\rho = 500$) is competitive on wall-clock among non-gradient methods: variable splitting decouples the $k \times k$ $H$-update from the clamp $W$-update. Randomized LAI-SymNMF~\cite{hayashi2024randomized} converges but is slow and degrades under $k = 25$.

\paragraph{Family 7: Deep Unfolding.}
SymNMF-Net (5~blocks) fails on every $n = 10^4$ run: per-matrix training removes the amortization benefit, and training is two orders of magnitude slower than a direct solve even at small~$n$.

\phantomsection
\subsubsection*{Robustness check: fixed-rank ($k = 25$) stress test}\label{snmf:k25}
We re-run Phase~1 at fixed $k = 25$ for $n \in \{10^2, 10^3, 10^4\}$ (Section~\ref{snmf:matrices}): deliberately under-ranked at $n = 10^4$ (true rank $\approx 101$).
All six Phase~2 methods retain convergence and stay fast.
Several methods that succeed at $k = \lfloor\sqrt{n}\rfloor$ degrade: plain (non-SVRG) block and row stochastic AdaGrad fail every $n = 10^4$ run, and Adan, NAdam, LAI-SymNMF, and MU each lose multiple runs.
The SVRG hybrids stay robust where their plain counterparts fail; we do not repeat this check in Phase~2.

\subsubsection*{Phase~2 selection}

Phase~2 uses the six methods under 4\,s at $n = 10^4$ (dashed rule in Table~\ref{tab:phase1}): ADMM, RMSprop, Block-SVRG \textsc{AdaptGrow}, AdaGrad, Row-Stochastic SVRG, and Piecewise AdaGrad.
Three of them (Row-Stochastic SVRG, Block-SVRG \textsc{AdaptGrow}, and Piecewise AdaGrad) target effects that the Phase~1 per-iteration overhead floor hides (cheaper per-step gradients; accumulator resets on long runs). We test whether those pay off once the dense step dominates at $n \ge 10^5$.

\subsection{Softplus Reparameterization Analysis}\label{results:logspace}

Section~\ref{snmf:logspace} introduced softplus as an alternative to projection for $H \ge 0$. We quantify the effect on representative adaptive and second-order solvers.

\paragraph{Head-to-head comparison.}
Table~\ref{tab:logspace} reports strict convergence under projected vs.\ LogSpace parameterizations. The effect is scale-dependent and ordered by optimizer memory length. At small scale softplus does not hurt, and L-BFGS converges in LogSpace through $n \le 10^3$ while it fails in projected space. By $n = 10^4$, LogSpace fails for \emph{long-memory} adaptive methods (AdaGrad from $n = 10^3$; Adam at $n = 10^4$), while short-memory RMSprop converges at every scale. Where RMSprop converges in LogSpace it also reaches lower $E_t$ ($28$--$130\times$ on correlation): short-memory methods gain from the smooth surface; long-memory accumulators do not at this scale.

\begin{table}[ht]
\centering
\caption{Projected (P) vs.\ softplus/LogSpace (L) convergence (runs out of 6 per scale; 3 seeds $\times$ 2 types, $k = \lfloor\sqrt{n}\rfloor$; GB200). LogSpace enables L-BFGS; long-memory AdaGrad/Adam fail at scale; RMSprop is unchanged.}\label{tab:logspace}
\small
\begin{tabular}{l rr rr rr}
\toprule
& \multicolumn{2}{c}{$n = 10^2$} & \multicolumn{2}{c}{$n = 10^3$} & \multicolumn{2}{c}{$n = 10^4$} \\
\cmidrule(lr){2-3}\cmidrule(lr){4-5}\cmidrule(lr){6-7}
\textbf{Solver} & \textbf{P} & \textbf{L} & \textbf{P} & \textbf{L} & \textbf{P} & \textbf{L} \\
\midrule
RMSprop & 6 & 6 & 6 & 6 & 6 & 6 \\
Adam    & 6 & 6 & 6 & 6 & 6 & 0 \\
AdaGrad & 6 & 6 & 6 & 0 & 6 & 0 \\
Newton  & 6 & 6 & 6 & 5 & 3 & 0 \\
L-BFGS  & 0 & 6 & 0 & 6 & 0 & 0 \\
\bottomrule
\end{tabular}
\end{table}

The softplus chain rule injects a non-stationary factor $\sigma(\Theta_t) = \operatorname{sigmoid}(\Theta_t)$ into every gradient. Longer second-moment memory averages over a wider range of $\sigma$, so adaptive rates go stale; the table pattern tracks memory length and worsens with scale. Second-order methods have no such accumulator, and L-BFGS converges more readily on the smoother surface at small scale. Softplus still costs at scale: as $\Theta$ spreads, $\sigma(\Theta)$ and $\sigma'(\Theta)\to 0$ poorly scale Hessian estimates, so Newton and L-BFGS also fail in LogSpace by $n = 10^4$ (Newton from partial projected success).

The ordering matches memory length, though we have not run a controlled $\beta_2/\alpha$ sweep. Four of the six Phase~2 methods are long-memory AdaGrad-family solvers, so Phase~2 uses projected space. LogSpace RMSprop converges with lower reconstruction error; we leave that configuration for $n \ge 10^5$ untested.

\subsection{Phase 2: Large-Scale Feasibility ($n \ge 10^5$)}\label{results:phase2}

Phase~1 measured reliability and cost at $n \le 10^4$. Phase~2 runs the six selected methods on $n \in \{10^5,\, 10^6\}$: all six at $n = 10^5$ and the five AdaGrad-family methods at $n = 10^6$ (ADMM is omitted at that scale; it is the slowest on TPDM at $n = 10^5$). Each solver uses its best local sweep around the Phase~1 anchor (Table~\ref{tab:phase2_config}). At $n = 10^5$ both matrix types use $k = \lfloor\sqrt{n}\rfloor = 316$; at $n = 10^6$ correlation uses $k = 1000$, while TPDM uses $k = 400$ (its rank ceiling $n_{\mathrm{exc}} = 400$; Section~\ref{snmf:spectral}), so the two become distinct problems.

Row-Stochastic SVRG evaluates gradients on a row fraction between full-gradient snapshots. Block-SVRG \textsc{AdaptGrow} uses dense cuBLAS sub-blocks of relative size $\phi_t$, grown from $\phi_0$ toward~$1$ (Table~\ref{tab:phase2_config}).

At $n = 10^5$, $S$ is 40\,GB in fp32 and fits on one GB200 GPU, but the $O(n^2 k)$ product $SH$ is sharded across four GPUs of one node (${\sim}3.6\times$ per-iteration speedup, applied uniformly). At $n = 10^6$, $S$ is 4\,TB and is row-sharded across 16 nodes (64 GPUs; NCCL): each node holds a row block of $S$ and a full replica of $H$; communication is an \texttt{all\_gather} of $SH$ and an \texttt{all\_reduce} of its gradient ($O(nk)$ each), so replicas stay in sync.

\paragraph{Results at $n = 10^5$.}
All six solvers in Table~\ref{tab:phase2} converge on both types at $k = 316$, all three seeds, at the nominal gate $\lVert \nabla_{\mathrm{proj}}\rVert_F/(nk) < 10^{-4}$.
The five first-order methods share a common residual of $E_t \approx 1.27\times10^{-2}$ on correlation and $7.2\times10^{-4}$ on TPDM: once the gradient gate is met, the residual is set by the problem, not the optimizer.
ADMM reaches a lower floor of $8.1\times10^{-4}$ on correlation and $1.6\times10^{-6}$ on TPDM by splitting variables and avoiding the projected step-size ceiling near the boundary.

Wall time tracks the dense product $SH$. Piecewise AdaGrad is fastest on both types, then AdaGrad; both finish in under $100$ iterations, so full-batch steps beat the stochastic variants. ADMM is second on correlation and slowest on TPDM. Correlation uses larger steps than TPDM throughout, as in Table~\ref{tab:phase2_config}, matching the different conditioning in Table~\ref{tab:spectral}.

\begin{table}[t]
\centering
\small
\setlength{\tabcolsep}{5pt}
\caption{Phase~2 wall-clock and iterations to convergence on GB200 (fp32; means over seeds $\{7, 42, 99\}$, all $3/3$). Entries are wall\,(iters); seconds at $n = 10^5$ ($k = 316$, four GPUs), minutes at $n = 10^6$ (CORR $k = 1000$, TPDM $k = 400$, 16 nodes (64 GPUs); $\tau_g = 10^{-3}$ for TPDM at $n = 10^6$, else $10^{-4}$; Section~\ref{snmf:convergence}). ADMM only at $n = 10^5$. Bold: fastest per scale and type. $E_t$ in text.}
\label{tab:phase2}
\smallskip
\begin{tabular}{@{}l cc cc@{}}
\toprule
 & \multicolumn{2}{c}{Correlation} & \multicolumn{2}{c}{TPDM} \\
\cmidrule(lr){2-3}\cmidrule(l){4-5}
Solver & $n{=}10^5$ & $n{=}10^6$ & $n{=}10^5$ & $n{=}10^6$ \\
\midrule
AdaGrad & 7.9\,(82) & \textbf{2.0}\,(72) & 9.3\,(99) & 4.8\,(395) \\
Piecewise AdaGrad & \textbf{6.3}\,(72) & 37.1\,(1335) & \textbf{8.4}\,(99) & 5.7\,(475) \\
Row-Stoch.\ SVRG & 11.3\,(182) & 3.1\,(152) & 12.9\,(209) & 6.7\,(749) \\
Block-SVRG \textsc{AdaptGrow} & 13.0\,(122) & 2.9\,(122) & 12.4\,(112) & \textbf{4.0}\,(409) \\
RMSprop & 28.2\,(332) & 2.9\,(102) & 15.3\,(172) & 8.7\,(715) \\
ADMM & 6.7\,(102) & --- & 19.6\,(315) & --- \\
\bottomrule
\end{tabular}
\end{table}

\begin{table}[t]
\centering
\small
\setlength{\tabcolsep}{6pt}
\caption{Phase~2 configurations (local sweep around the Phase~1 anchor; three seeds). Bold: fastest run in Table~\ref{tab:phase2}. Correlation uses larger steps than TPDM.}
\label{tab:phase2_config}
\smallskip
\begin{tabular}{@{}ll ll@{}}
\toprule
Solver & Type & $n = 10^5$ & $n = 10^6$ \\
\midrule
\multirow{2}{*}{AdaGrad} & CORR & $\eta{=}96$ & $\bm{\eta{=}384}$ \\
 & TPDM & $\eta{=}32$ & $\eta{=}32$ \\
\multirow{2}{*}{Piecewise AdaGrad} & CORR & $\bm{\eta{=}128,\,c{=}10}$ & $\eta{=}0.75,\,c{=}200$ \\
 & TPDM & $\bm{\eta{=}32,\,c{=}10}$ & $\eta{=}0.5,\,c{=}20$ \\
\multirow{2}{*}{Row-Stoch.\ SVRG} & CORR & $\eta{=}128,\,|I|/n{=}0.75,\,s{=}10$ & $\eta{=}256,\,|I|/n{=}0.75,\,s{=}5$ \\
 & TPDM & $\eta{=}64,\,|I|/n{=}0.75,\,s{=}10$ & $\eta{=}16,\,|I|/n{=}0.75,\,s{=}5$ \\
\multirow{2}{*}{Block-SVRG \textsc{AdaptGrow}} & CORR & $\eta{=}64,\,\phi_0{=}0.75,\,s{=}10$ & $\eta{=}256,\,\phi_0{=}0.5,\,s{=}10$ \\
 & TPDM & $\eta{=}24,\,\phi_0{=}0.75,\,s{=}10$ & $\bm{\eta{=}20,\,\phi_0{=}0.75,\,s{=}5}$ \\
\multirow{2}{*}{RMSprop} & CORR & $\eta{=}0.4,\,\mu{=}0.5,\,\alpha{=}0.999$ & $\eta{=}32,\,\mu{=}0.25,\,\alpha{=}0.99$ \\
 & TPDM & $\eta{=}0.4,\,\mu{=}0.25,\,\alpha{=}0.999$ & $\eta{=}1.6,\,\mu{=}0.75,\,\alpha{=}0.999$ \\
\multirow{2}{*}{ADMM} & CORR & $\rho{=}500$ & --- \\
 & TPDM & $\rho{=}500$ & --- \\
\bottomrule
\end{tabular}
\end{table}

\paragraph{Results at $n = 10^6$.}
On correlation at $k = 1000$, all five AdaGrad-family methods converge on all three seeds to $E_t \approx 4.6\times10^{-3}$, lower than at $n = 10^5$ because larger $k$ retains more of the noise bulk.
AdaGrad is fastest, with bit-identical iteration counts across seeds; the stochastic variants and RMSprop cluster next; Piecewise AdaGrad is far slower, its reset cadence poorly matched to a long smooth descent.
The large $\eta$ in Table~\ref{tab:phase2_config} rescales $\eta\,g/\sqrt{G}$ globally and offsets the shared $O(1/\sqrt{t})$ factor in $1/\sqrt{G}$, while leaving relative per-coordinate weights $1/\sqrt{G_i}$ intact.

On TPDM at $k = 400$, the factorization is essentially exact in fp32, so $E_t \to 0$ and the projected gradient is precision-limited near $10^{-3}$. We certify at $\tau_g = 10^{-3}$ as in Sections~\ref{snmf:loss} and~\ref{snmf:convergence}.
All five methods converge on all three seeds. Block-SVRG \textsc{AdaptGrow} is fastest, then AdaGrad; Piecewise AdaGrad and RMSprop show the widest seed spreads, and RMSprop needs $\alpha = 0.999$.

\paragraph{Why the ordering splits by matrix type.}
AdaptGrow's sub-block step is cheaper than full-batch AdaGrad's, so the faster method is whichever keeps a comparable iteration count (Table~\ref{tab:spectral}).

On correlation, $\gamma_{k+1} \approx 39$ separates a rank-$1000$ subspace that AdaGrad captures in a short run, similar in length to $n = 10^5$. Cheap stochastic steps do not amortize, and with energy in only ${\sim}4$ factors ($r_{\mathrm{eff}} \approx 4$) a random sub-block is high-variance, so AdaptGrow needs substantially more iterations and full-batch AdaGrad remains faster.

On TPDM, $\gamma_k \approx 1$ and large $\kappa^+$ force a long descent that lengthens with~$n$. Energy concentrates in one common factor ($r_{\mathrm{eff}} \to 1$), so rows are near-collinear and a random sub-block is low-variance: AdaptGrow matches AdaGrad's iteration count at lower per-step cost, and is correspondingly slower on the short $n = 10^5$ runs.

At $n = 10^6$, AdaGrad suits correlation-type spectra and Block-SVRG \textsc{AdaptGrow} suits TPDM-type spectra. At $n = 10^5$ every run is short, so full-batch Piecewise or AdaGrad is fastest on both. Phase~2 shows SymNMF remains practical when $S$ exceeds single-GPU memory.

\section{Clustering Alternatives}\label{clustering}

SymNMF's factorization $S \approx HH^\top$ is already a soft clustering: the rows of $H$ give each point's membership across the $k$ groups, and $\arg\max_j H_{ij}$ recovers a hard label.

When only hard labels are needed, a direct clustering algorithm can bypass the factorization. Both inputs are unit-diagonal (Section~\ref{estimation}): a matrix of unit vectors for the TPDM, and likewise for absolute correlation up to the elementwise $|\cdot|$, which can mildly break PSD at scale (Section~\ref{snmf:spectral}). Theorems~\ref{thm:snmf_kmeans}--\ref{thm:snmf_spkmeans} link orthogonality-constrained SymNMF on that matrix to classical and, approximately, spherical K-means on the latent unit vectors that produce it. Given only~$S$, we run spherical K-means on the rows of~$S$ as the practical kernel-style baseline; the unconstrained Frobenius SymNMF we optimize further drops orthogonality. We develop both equivalences below, then compare efficiency in Section~\ref{clustering:results}.

\subsection{Theoretical Equivalences}\label{clustering:theory}

Here $K$ denotes the number of clusters and $k$ the summation index.

\subsubsection{K-means as constrained SymNMF}

Given \( n \) observations \( \bm{x}_1, \dots, \bm{x}_n \in \mathbb{R}^m \), K-means~\cite{macqueen1967} partitions them into \( K \) clusters \( C_1,\dots, C_K \) by minimizing the within-cluster variance:
\[
J_K = \sum_{k=1}^{K}\sum_{i \in C_k} \lVert \bm{x}_i - \bm{m}_k \rVert^2,
\]
where \( \bm{m}_k = \frac{1}{n_k}\sum_{i \in C_k} \bm{x}_i \) is the centroid with cardinality \( n_k \).
The connection to SymNMF was formalized by \cite{ding2005} and \cite{kuang2015symnmf}:

\begin{theorem}\label{thm:snmf_kmeans}
Let \( \bm{X} = [\bm{x}_1 \cdots \bm{x}_n] \in \mathbb{R}^{m \times n} \) be the data matrix, \( \bm{M} = [\bm{m}_1 \cdots \bm{m}_K] \in \mathbb{R}^{m \times K} \) the centroid matrix, and \( \bm{B} \in \{0,1\}^{n \times K} \) the binary assignment matrix with exactly one non-zero entry per row. Define \( \bm{D} \coloneqq \bm{B}^\top \bm{B} = \diag(n_1, \dotsc, n_K) \). Then:
\[
\bm{B} = \argmin{\bm{B},\bm{M}} J_K \quad \text{where } \bm{M} = \bm{X}\bm{B}\bm{D}^{-1}
\]
\[
\bm{L} = \argmin{\substack{\bm{L}^\top \bm{L} = \bm{I}_K \\ \bm{L} \geq 0}} \lVert \bm{X}^\top \bm{X} - \bm{L}\bm{L}^\top \rVert_F^2
\]
on the set of assignment-induced matrices $\bm{L} = \bm{B}\bm{D}^{-1/2}$ ($\bm{L}$ has exactly one nonzero entry per row).
\end{theorem}

The equivalence follows from the substitution \( \bm{L} = \bm{B}\bm{D}^{-1/2} \), which satisfies both non-negativity and orthogonality \(\bm{L}^\top \bm{L} = \bm{I}_K\); the proof is in the appendix. The continuous relaxation, allowing arbitrary $\bm{L} \ge 0$ with $\bm{L}^\top\bm{L} = \bm{I}_K$, is a superset of this discrete family and links to SymNMF \cite{ding2005,kuang2015symnmf}. Constrained SymNMF thus recovers the K-means solution through the structure of \(\bm{L}\): K-means is SymNMF with an added orthogonality constraint.

\subsubsection{Spherical K-means as constrained SymNMF}

Spherical K-means \cite{dhillon2001,hornik2012} optimizes cosine similarity rather than Euclidean distance, the natural objective for data on the unit sphere such as our standardized inputs (Section~\ref{estimation}). It partitions the data into \( K \) clusters by minimizing angular dissimilarity:
\[
S_K = \sum_{k=1}^{K}\sum_{i \in C_k} \left(1 - \langle \tilde{\bm{x}}_i, \tilde{\bm{c}}_k \rangle\right),
\]
where \( \tilde{\bm{x}}_i = \bm{x}_i / \lVert \bm{x}_i \rVert \) and \( \tilde{\bm{c}}_k = \bm{c}_k / \lVert \bm{c}_k \rVert \) are the $\ell_2$-normalized observations and centroids \citep{hornik2012}.

\begin{theorem}\label{thm:snmf_spkmeans}
Let \( \tilde{\bm{X}} = [\tilde{\bm{x}}_1 \cdots \tilde{\bm{x}}_n] \in \mathbb{R}^{m \times n} \) be the normalized data matrix with \( \lVert \tilde{\bm{x}}_i \rVert = 1 \), and \( \tilde{\bm{B}} \in \{0,1\}^{n \times K} \) the binary assignment matrix with \( \tilde{\bm{D}} \coloneqq \tilde{\bm{B}}^\top \tilde{\bm{B}} = \diag(n_1, \dotsc, n_K) \). Define \( \tilde{\bm{L}} = \tilde{\bm{B}}\tilde{\bm{D}}^{-1/2} \) and let \( \bm{s}_k = \sum_{i \in C_k} \tilde{\bm{x}}_i \). Then:
\begin{enumerate}
\item Minimizing \( S_K \) is equivalent to maximizing \( \sum_{k=1}^{K} \lVert \bm{s}_k \rVert \).
\item Minimizing the orthogonally-constrained SymNMF objective \( \lVert \tilde{\bm{X}}^\top \tilde{\bm{X}} - \tilde{\bm{L}}\tilde{\bm{L}}^\top \rVert_F^2 \) is equivalent to maximizing \( \sum_{k=1}^{K} \lVert \bm{s}_k \rVert^2 / n_k \).
\end{enumerate}
\end{theorem}

The two objectives differ only by cardinality weighting: spherical K-means maximizes \( \sum_k \lVert \bm{s}_k \rVert \), while constrained SymNMF maximizes \( \sum_k \lVert \bm{s}_k \rVert^2 / n_k \). For balanced partitions (\( n_k \approx n/K \)) the weights are nearly common and the objectives align, so SymNMF on \( \tilde{\bm{X}}^\top \tilde{\bm{X}} \) acts as a continuous relaxation of spherical K-means that tends to recover the same structure \citep{ding2005,kuang2015symnmf}; the alignment loosens for strongly unbalanced partitions.

The distinction from Theorem~\ref{thm:snmf_kmeans} is that the spherical centroid is the normalized sum \( \tilde{\bm{c}}_k = \bm{s}_k / \lVert \bm{s}_k \rVert \), not the mean \( \bm{s}_k / n_k \). That blocks the idempotent projection argument that gives an exact algebraic equivalence in the Euclidean case. Both objectives penalize dispersed clusters and reward tightly aligned groups, so SymNMF on the matrix remains a surrogate for angular clustering.

Both inputs are unit-diagonal, so spherical K-means is the matched discrete objective; classical K-means would suit an un-standardized input such as raw covariance. Given only~$S$, we run it on the rows of~$S$; unconstrained Frobenius SymNMF further drops orthogonality.

\subsection{Computational Comparison}\label{clustering:results}
We benchmark spherical K-means (Lloyd) on the rows of $S$ against the Phase~2 SymNMF solvers under a matched setup: the same PyTorch + NCCL row-sharding and the same hardware allocation as Section~\ref{results:phase2}. Both SymNMF and K-means are reported in fp32; SymNMF residuals are additionally unchanged under TF32 (Section~\ref{snmf:compute}), whereas K-means on TPDM is sensitive to that setting, so we keep TF32 off for the K-means runs.

The convergence criteria differ, so iteration counts are not comparable. SymNMF must clear the three-criterion rule~\eqref{eq:conv_and}; K-means stops on objective stagnation alone, a relative drop in $\sum_i (1 - \cos)$ below $10^{-4}$, as soon as the Lloyd assignments stabilize, and all runs converged. K-means thus certifies a weaker condition on a simpler problem, and part of its speed advantage is the lighter stopping rule. We compare wall-clock and the silhouette of the resulting partition as a label-free quality summary.

\begin{table}[t]
\centering\footnotesize
\caption{Spherical K-means (Lloyd on the rows of $S$) across scales and types. Hardware matches the SymNMF benchmarks: single GPU at $n \le 10^4$, four GPUs at $n = 10^5$, 16 nodes (64 GPUs) at $n = 10^6$. True fp32 with TF32 off; means over seeds $\{7, 42, 99\}$, all converged. Silhouette is label-free. At $n = 10^6$ TPDM uses $k = 400$ (Section~\ref{snmf:spectral}).}
\label{tab:kmeans}
\smallskip
\begin{tabular}{l r r r r r}
\toprule
Type & $n$ & $k$ & Silhouette & iters & wall (s) \\
\midrule
\multirow{5}{*}{CORR}
 & 100 & 10 & 0.78 & 5 & 0.015 \\
 & 1\,000 & 31 & 0.71 & 6 & 0.040 \\
 & 10\,000 & 100 & 0.57 & 12 & 0.146 \\
 & 100\,000 & 316 & 0.48 & 12 & 1.93 \\
 & 1\,000\,000 & 1\,000 & 0.32 & 17 & 23.7 \\
\midrule
\multirow{5}{*}{TPDM}
 & 100 & 10 & 0.75 & 3 & 0.13 \\
 & 1\,000 & 31 & 0.67 & 9 & 0.044 \\
 & 10\,000 & 100 & 0.60 & 12 & 0.145 \\
 & 100\,000 & 316 & 0.37 & 17 & 2.20 \\
 & 1\,000\,000 & 400 & ${\approx}0$ & 4 & 6.9 \\
\bottomrule
\end{tabular}
\end{table}

\paragraph{Cost.}
A Lloyd step and a SymNMF gradient step share the same $O(n^2 k)$ product of $S$ with an $n \times k$ factor, so the wall-clock gap is iteration count: K-means finishes in at most $17$ iterations, versus tens to hundreds for the Phase~2 leaders (Table~\ref{tab:phase2}). Where both produce a meaningful partition, K-means is several times faster at $n \ge 10^5$. When only hard labels are needed, that is enough; the soft embedding $H$ still requires the factorization.

\paragraph{A spectral limitation of the hard baseline.}
Table~\ref{tab:kmeans} also shows a limitation of spherical K-means. On correlation the silhouette degrades gracefully with scale; on TPDM it tracks correlation through moderate~$n$ and then collapses at $n = 10^6$, where Lloyd stops almost immediately on a meaningless partition.

The cause is the effective-rank collapse of Table~\ref{tab:spectral} ($r_{\mathrm{eff}} = \operatorname{tr}(S)/\lambda_1 \to 1.1$ at $n = 10^6$). Decompose the symmetric PSD input as $S = \sum_j \lambda_j \bm u_j \bm u_j^\top$ with $\lambda_1 \ge \lambda_2 \ge \cdots \ge 0$; its $i$-th row is $\bm s_i = S\bm e_i = \sum_j \lambda_j u_{ji}\,\bm u_j$. Writing $\theta_i$ for the angle between $\bm s_i$ and the leading eigenvector $\bm u_1$, and using $\bm u_1^\top \bm s_i = \lambda_1 u_{1i}$ with $\sum_i u_{1i}^2 = 1$,
\begin{equation}
\frac{\sum_i \lVert \bm s_i\rVert^2 \sin^2\theta_i}{\sum_i \lVert \bm s_i\rVert^2}
= \frac{\lVert S\rVert_F^2 - \lambda_1^2}{\lVert S\rVert_F^2}
= \frac{\sum_{j\ge2}\lambda_j^2}{\sum_j \lambda_j^2}
\;\le\; (r_{\mathrm{eff}} - 1)^2,
\label{eq:angular_collapse}
\end{equation}
where the bound follows from $\sum_{j\ge2}\lambda_j^2 \le (\operatorname{tr}S - \lambda_1)^2 = \lambda_1^2 (r_{\mathrm{eff}}-1)^2$ and $\lVert S\rVert_F^2 \ge \lambda_1^2$. At $n = 10^6$ this caps the norm-weighted root-mean-square of $\sin\theta_i$ at $r_{\mathrm{eff}} - 1 = 0.1$.

Moreover the TPDM is entrywise positive and irreducible, so Perron--Frobenius gives a unique strictly positive $\bm u_1$. Every row then has a positive projection onto $\bm u_1$, and after unit normalization all rows concentrate on that same direction. The spherical K-means objective is near-zero for every partition, so the assignment is non-identifiable and the silhouette vanishes, for any $k$ and any optimizer. The bound~\eqref{eq:angular_collapse} is monotone in $r_{\mathrm{eff}} - 1$, matching Table~\ref{tab:spectral}: hard angular clustering fails when a single common factor dominates.

SymNMF avoids it by fitting magnitudes rather than angles, reconstructing the same matrix to the precision floor by $HH^\top$ ($E_t \to 0$; Section~\ref{results:phase2}). The soft factorization stays necessary both for the embedding and in this near-rank-1 regime.

\section{Related Work}\label{related}

\paragraph{Extreme value theory and dependence modeling.}
The TPDM of \cite{Cooley2019} admits exact (PCA-like) factorizations for regularly varying vectors; we use SymNMF for an approximate factorization at scales where those methods are impractical. Multivariate regular variation follows \cite{resnick2010heavy,resnick19872014extreme}, with extreme-value background in \cite{Coles2001,de2010extreme}. Tail dependence in financial risk is motivated by regulation and portfolio construction \cite{BCBS2011,DeLuca2011Tail,Lohre2020Hierarchical}; Pareto(2) margins are standard \cite{Hult2007,Cheng2024}.

\paragraph{NMF, SymNMF, and distributed solvers.}
NMF was popularized by \cite{lee1999learning,lee2001algorithms}; SymNMF and its link to graph clustering by \cite{kuang2012symmetric,kuang2015symnmf}, with identifiability in \cite{gillis2020nmf}. Lines we benchmark include dampened MU \cite{he2011symmetric}, ANLS/HALS \cite{kim2014algorithms,kimpark2008}, PGNCG and the CPU/MPI library PLANC \cite{eswar2020pgncg,eswar2021planc}, ADMM in non-convex and SymNMF settings \cite{boyd2011distributed,deng2012nonconvex,lu2017stochastic,hong2016convergence}, and deep unfolding \cite{li2022symncf}. MPI-FAUN \cite{kannan2018mpifaun} is the other mature CPU-cluster NMF stack. We target single- and multi-GPU settings and compare these families under one convergence criterion, rather than a wall-clock head-to-head against MPI on CPU.

\paragraph{Clustering via SymNMF.}
Orthogonality-constrained SymNMF is equivalent to K-means on the latent factors \cite{ding2005,kuang2015symnmf}; spherical K-means replaces Euclidean distance by cosine similarity \cite{dhillon2001,hornik2012}. We use Lloyd on the rows of~$S$ as the discrete baseline when only hard labels are needed (Section~\ref{clustering}), and keep unconstrained Frobenius SymNMF when a soft embedding is required or angular structure collapses.

\paragraph{Convergence criteria.}
Projected-gradient stopping for constrained NMF follows \cite{lin2007projected}. Our three-criterion AND rule (Section~\ref{snmf:convergence}) adds a loss gate and per-element KKT normalization.

\paragraph{Constraint handling.}
Smooth bijections for $H \ge 0$ connect to mirror descent and multiplicative weights \cite{nemirovsky1983problem,amid2020}; common maps are $e^\Theta$, $\Theta^2$, and softplus \cite{subramani2024rethinking}. The interaction of the bijection with optimizer memory length (Section~\ref{results:logspace}) appears new in the NMF setting.

\paragraph{Adaptive gradients and stochastic SymNMF.}
AdaGrad's growing accumulator \cite{duchi2011adagrad} saturates on long runs; RMSprop and Adam forget via EMAs \cite{hinton2012rmsprop,kingma2014adam,mukkamala2017variants}. The decaying step size, not the direction, often drives AdaGrad's behavior \cite{agarwal2020grafting}. Scheduled remedies include accumulator decay \cite{deeprec_adagraddecay}, cosine learning-rate restarts \cite{loshchilov2017sgdr}, and non-monotone accumulators \cite{defazio2022gradagrad}. Piecewise AdaGrad instead resets on objective stagnation, recalibrates the accumulator per entry, and leaves the learning rate untouched (Section~\ref{snmf:perf}).

For cheaper steps we adapt SVRG \cite{johnson2013svrg} to row and block entry sampling under AdaGrad. Randomized SymNMF of \cite{hayashi2024randomized} samples inside constrained least-squares subproblems; we sample the gradient itself and feed it to an AdaGrad accumulator. Growing-batch follows the norm test and sample-complexity bounds of \cite{friedlander2012hybrid,byrd2012sample}; we adapt them to non-convex SymNMF, where block sampling creates an accumulator heterogeneity that SVRG must repair (Section~\ref{snmf:perf}).

\section{Conclusions}\label{conclusions}

Removing $n \times n$ intermediates makes dense SymNMF practical through $n = 10^6$. Five AdaGrad-family methods converge at that scale; ADMM reaches a lower reconstruction floor at $n = 10^5$ than the projected first-order methods, but is not among the fastest there and was not scaled further.

A central finding is that diagonal AdaGrad, standard in deep learning but not previously studied for dense SymNMF, is the method that makes this problem tractable at scale: we apply it and its adaptive relatives through $n = 10^6$ and find that they alone occupy the reliability--efficiency frontier. AdaGrad is the simplest member of that family and the default when runs stay short. Piecewise AdaGrad is fastest at $n = 10^5$, but its accumulator resets become fragile on long smooth descents. When iteration counts grow, Block-SVRG \textsc{AdaptGrow}'s cheaper sub-block step pays off once it matches full-batch iterations; because it grows the batch from the same stagnation signal, it can start stochastic and fall back to full-batch AdaGrad without choosing a regime in advance. Projection is preferred for long-memory adaptive methods at scale; softplus helps short-memory and some second-order methods only at small~$n$.

When only hard labels are needed and angular structure is present, spherical K-means is cheaper than SymNMF-then-argmax. As the effective rank falls toward one, that baseline fails: the same near-rank-1 regime that breaks angular clustering makes AdaptGrow's random sub-block a low-variance, cheap gradient estimate, while the soft factorization still fits.

Natural next steps are streaming SymNMF for rolling windows and a profiled multi-node scaling study. Open theoretical questions remain: when the spectral gap selects the faster solver, when low effective rank makes sub-sampling faithful, and how the learning rate should grow with~$n$.

\section*{Acknowledgements}
We thank Siddharth Samsi and Saleh Ashkboos for help in reviewing the paper, and Ioana Boier for helpful advice and broader guidance throughout this work.

\clearpage
\renewcommand{\thesection}{\Alph{section}}
\section*{Appendix}\label{appendix}
\setcounter{section}{1}
\renewcommand{\thesubsection}{\Alph{subsection}}

\subsection{Proofs}\label{proofs}
{
\setlength{\abovedisplayskip}{0.5\baselineskip}
\setlength{\belowdisplayskip}{0.5\baselineskip}
\setlength{\abovedisplayshortskip}{0.25\baselineskip}
\setlength{\belowdisplayshortskip}{0.25\baselineskip}
\setlength{\topsep}{0.35\baselineskip}
\setlength{\partopsep}{0pt}

\paragraph{Regular variation of the factor model.}
The following results underlie the TPDM construction of Section~\ref{tpdm}; \(\Vert\cdot\Vert_2\) is the Euclidean norm, matching the polar decomposition of Section~\ref{regvar}.

\begin{proposition}\label{prop:sum} (Theorem 7.4 in \cite{resnick2010heavy}) Let \(\bm{X}_1, \bm{X}_2 \in RV_{+}^n(\alpha)\) be independent random vectors with a \emph{common} normalizing sequence \(\{b_N\}\) such that \(N\mathbb{P}(b_N^{-1} \bm{X}_1 \in \cdot) \xrightarrow{v} \nu_{\bm{X}_1}(\cdot)\) and \(N\mathbb{P}(b_N^{-1} \bm{X}_2 \in \cdot) \xrightarrow{v} \nu_{\bm{X}_2}(\cdot)\). Then \(\bm{X}_1 + \bm{X}_2 \in RV_{+}^n(\alpha)\) and
\[
N\mathbb{P}(b_N^{-1} (\bm{X}_1 + \bm{X}_2 ) \in \cdot) \xrightarrow{v} \nu_{\bm{X}_1}(\cdot) + \nu_{\bm{X}_2}(\cdot) \text{.}
\]
\end{proposition}

\begin{proposition}\label{prop:prod} (See e.g.\ Proposition~5.3 in \cite{resnick2010heavy}) Let \(\bm{X} \in RV_{+}^n(\alpha)\) such that \(N\mathbb{P}(b_N^{-1} \bm{X} \in \cdot) \xrightarrow{v} \nu_{\bm{X}}(\cdot)\) in \(M_{+}\left( [0,\infty]^n \backslash \{0\} \right)\). Then for \(a \in \mathbb{R}_{+}\),
\[
N\mathbb{P}(b_N^{-1} (a \bm{X}) \in \cdot) \xrightarrow{v} a^{\alpha} \nu_{\bm{X}}(\cdot) \text{.}
\]
\end{proposition}

\begin{proof}[Proof of Proposition~\ref{prop:prod}]
For $a = 0$ the conclusion is trivial.
For $a > 0$, and any $C \subset \left( [0,\infty]^n \backslash \{0\} \right)$, using the scaling property of the limiting measure,
\[
N\mathbb{P}(b_N^{-1} (a \bm{X} ) \in C) =  N\mathbb{P}(b_N^{-1} \bm{X} \in a^{-1}C)  \xrightarrow{v} \nu_{\bm{X} }(a^{-1}C) = a^{\alpha}\nu_{\bm{X} }(C) \text{.}
\]
\end{proof}

\begin{corollary}\label{cor:scalar} Let \(\bm{a} \in \mathbb{R}_{+}^n\), where \(\max_{i = 1,\dots,n} a_{i} > 0\). Let \(Z\) be a regularly varying \(\alpha\) random variable with \(\{ b_N \}_{N \geq 1} \subset \mathbb{R}\) such that \( N\mathbb{P}(b_N^{-1} Z > z) \rightarrow z^{-\alpha}, z > 0 \).
Then \( \bm{a} Z \in RV_{+}^n(\alpha)\) and when normalized by \(\{ b_N \}\) has angular measure
\[
\mathcal{H}_{\bm{a} Z} (\cdot) = \Vert \bm{a} \Vert_2^{\alpha} \delta_{\bm{a} / \Vert \bm{a} \Vert_2 }(\cdot) \text{.}
\]
\end{corollary}
\begin{proof}[Proof of Corollary~\ref{cor:scalar}]
Theorem 6.1 in \cite{resnick2010heavy} guarantees that it is sufficient to prove convergence on sets $[0,\bm{x}]^{c}$, $\bm{x}>0$, to ensure convergence on $M_{+}\left( [0,\infty]^n \backslash \{0\} \right)$.
\begin{align*}
N\mathbb{P}(b_N^{-1} (\bm{a} Z ) \in [0, \bm{x}]^c)
 & = N\mathbb{P} \left( \bigcup_{i: a_i > 0} \{ b_N^{-1} (a_i  Z ) > x_i \} \right)
 = N\mathbb{P} \left( \bigcup_{i: a_i > 0} \{ Z  > b_N a_i^{-1} x_i \} \right) \\
 & = N\mathbb{P} \left( b_N^{-1}  Z  > \min_{i: a_i > 0} a_i^{-1} x_i  \right)
 \rightarrow  \left( \max_{i: a_i > 0} a_i x^{-1}_i  \right)^{\alpha}
 = \max_{i=1,\dots,n} a_i^{\alpha} x_i^{-\alpha} \\
 & = \Vert \bm{a} \Vert_2^{\alpha} \max_{i=1,\dots,n} x_i^{-\alpha} \left( a_i /  \Vert \bm{a} \Vert_2 \right)^{\alpha} \text{.}
\end{align*}
Then $\nu_{\bm{X}} ([0, \bm{x}]^c) = \displaystyle\int_{\mathbb{S}_{+}^{n-1}}  \max_{i=1,\dots,n} (w_i^{\alpha} / x_i^{\alpha}) d\mathcal{H}_{\bm{X}}(\bm{w})$, and we can identify
\[
\mathcal{H}_{\bm{a} Z}(\cdot)  =  \Vert \bm{a} \Vert_2^{\alpha} \delta_{\bm{a} / \Vert \bm{a} \Vert_2 }(\cdot) \text{.}
\]
\end{proof}

\begin{corollary}\label{cor} Let \(A = \left(\bm{a}_1, \dots,\bm{a}_k \right) \in \mathbb{R}_{+}^{n \times k}\) be a matrix with \(\max_{i = 1,\dots,n} a_{ij} > 0\), for all \(j = 1, \dots, k\) and let \(\bm{Z} = (Z_1, \dots, Z_k)^\top\) be a vector of independent and identically distributed regularly varying \(\alpha\) random variables with \(\{ b_N \}_{N \geq 1} \subset \mathbb{R}\) such that
\[
N\mathbb{P}(b_N^{-1} Z_j > z) \rightarrow z^{-\alpha} \text{, for all } j = 1,\dots,k \text{.}
\]
Then \(A \bm{Z} = \sum_{i=1}^{k} \bm{a}_i Z_i \in RV_{+}^n(\alpha)\) and, when normalized by \(\{ b_N \}\),
\[
\mathcal{H}_{A \bm{Z}} (\cdot) = \sum_{j=1}^{k} \Vert \bm{a}_j \Vert_2^{\alpha} \delta_{\bm{a}_j / \Vert \bm{a}_j \Vert_2 }(\cdot) \text{.}
\]
\end{corollary}
\begin{proof}[Proof of Corollary~\ref{cor}]
From Corollary~\ref{cor:scalar},
\(\mathcal{H}_{\bm{a}_j  Z_j } (\cdot) =  \Vert \bm{a}_j \Vert_2^{\alpha} \delta_{ \bm{a}_j  / \Vert \bm{a}_j \Vert_2 }(\cdot)\) for all \(j = 1,\dots,k\).
Using $A \bm{Z} = \sum_{j=1}^{k} \bm{a}_j Z_j $ and the additivity of limiting measures for independent regularly varying summands (Proposition~\ref{prop:sum}),
\[
\mathcal{H}_{A \bm{Z}}(\cdot) =\sum_{j=1}^{k} \mathcal{H}_{\bm{a}_j Z_j}(\cdot)  = \sum_{j=1}^{k} \Vert \bm{a}_j \Vert_2^{\alpha} \delta_{\bm{a}_j / \Vert \bm{a}_j \Vert_2 }(\cdot) \text{.}
\]
\end{proof}

\begin{proof}[Proof of Theorem~\ref{thm:snmf_kmeans}]
We have \(\bm{M} = \bm{X}\bm{B}\bm{D}^{-1}\), and the cost function can be rewritten as
\begin{align*}
J_K &= \sum_{k=1}^{K}\sum_{i \in C_k} \lVert \bm{x}_i - \bm{m}_k \rVert^2
= \lVert \bm{X} - \bm{M}\bm{B}^\top \rVert_F^2 \\
&= \lVert \bm{X} - \bm{X}\bm{B}\bm{D}^{-1}\bm{B}^\top \rVert_F^2
= \lVert \bm{X} (\bm{I}_n - \bm{B}\bm{D}^{-1}\bm{B}^\top) \rVert_F^2 \\
&= \operatorname{Tr}\left((\bm{I}_n - \bm{B}\bm{D}^{-1}\bm{B}^\top)^\top \bm{X}^\top \bm{X} (\bm{I}_n - \bm{B}\bm{D}^{-1}\bm{B}^\top)\right)
= \operatorname{Tr}\left((\bm{I}_n - \bm{B}\bm{D}^{-1}\bm{B}^\top)\bm{X}^\top \bm{X}\right) \\
&= \operatorname{Tr}(\bm{X}^\top \bm{X}) - \operatorname{Tr}(\bm{B}\bm{D}^{-1}\bm{B}^\top \bm{X}^\top \bm{X})
= \operatorname{Tr}(\bm{X}^\top \bm{X}) - \operatorname{Tr}\left((\bm{X}\bm{B}\bm{D}^{-1/2})^\top (\bm{X}\bm{B}\bm{D}^{-1/2})\right).
\end{align*}
Let \(\bm{L} \coloneqq \bm{B}\bm{D}^{-1/2} \in \mathbb{R}_+^{n \times K}\). Then \(\bm{L}^\top \bm{L} = \bm{D}^{-1/2}\bm{B}^\top \bm{B}\bm{D}^{-1/2} = \bm{I}_K\), so
\[
J_K = \operatorname{Tr}(\bm{X}^\top \bm{X}) - \operatorname{Tr}(\bm{L}^\top \bm{X}^\top \bm{X}\bm{L}).
\]
Under \(\bm{L}^\top \bm{L} = \bm{I}_K\) and \(\bm{L} \geq 0\),
\[
\min J_K \iff \max \operatorname{Tr}(\bm{L}^\top \bm{X}^\top \bm{X}\bm{L})
\iff \min \lVert \bm{X}^\top \bm{X} - \bm{L}\bm{L}^\top \rVert_F^2.
\]
\end{proof}

\begin{proof}[Proof of Theorem~\ref{thm:snmf_spkmeans}]
\emph{(1).}
For unit-norm data \(\lVert \tilde{\bm{x}}_i \rVert = 1\) and normalized centroids \(\tilde{\bm{c}}_k = \bm{s}_k / \lVert \bm{s}_k \rVert\),
\[
S_K = \sum_{k=1}^{K}\sum_{i \in C_k}\!\bigl(1 - \tilde{\bm{x}}_i^\top \tilde{\bm{c}}_k\bigr)
    = n - \sum_{k=1}^{K} \bm{s}_k^\top \frac{\bm{s}_k}{\lVert \bm{s}_k \rVert}
    = n - \sum_{k=1}^{K} \lVert \bm{s}_k \rVert.
\]
\emph{(2).}
Define \(\tilde{\bm{L}} := \tilde{\bm{B}}\tilde{\bm{D}}^{-1/2}\).
As in the proof of Theorem~\ref{thm:snmf_kmeans}, \(\tilde{\bm{L}}^\top \tilde{\bm{L}} = \bm{I}_K\) and \(\tilde{\bm{L}} \geq 0\).
Applying the Euclidean K-means trace expansion to unit-norm data (Theorem~\ref{thm:snmf_kmeans} with \(\tilde{\bm{X}}\) in place of \(\bm{X}\)),
\[
\operatorname{Tr}(\tilde{\bm{L}}^\top \tilde{\bm{X}}^\top \tilde{\bm{X}} \tilde{\bm{L}})
= \sum_{k=1}^{K} \frac{\lVert \bm{s}_k \rVert^2}{n_k}.
\]
Minimizing \(\lVert \tilde{\bm{X}}^\top \tilde{\bm{X}} - \tilde{\bm{L}}\tilde{\bm{L}}^\top \rVert_F^2\) under the orthogonality constraint is equivalent to maximizing this trace.

\emph{Gap between the two objectives.}
(1)~maximizes \(\sum_k \lVert \bm{s}_k \rVert\); (2)~maximizes \(\sum_k \lVert \bm{s}_k \rVert^2 / n_k\).
By Cauchy--Schwarz, \(\sum_k \lVert \bm{s}_k \rVert^2 / n_k \geq (\sum_k \lVert \bm{s}_k \rVert)^2 / n\), with equality iff \(\lVert \bm{s}_k \rVert \propto n_k\).
The two objectives therefore differ only through the cardinality weighting \(1/n_k\), which is nearly common across clusters for balanced partitions, leaving them closely aligned in practice.
\end{proof}
}

\subsection{Benchmark Data Generation}\label{app:matrices}

Storage for the dense $n \times n$ input scales as:

\begin{center}
\begin{tabular}{rrcc}
\toprule
\textbf{$n$} & \textbf{$k = \lfloor\sqrt{n}\rfloor$} & \textbf{$S$ (fp32)} & \textbf{$H$ (fp32)} \\
\midrule
100 & 10 & 40\,KB & 4\,KB \\
$1{,}000$ & 31 & 4\,MB & 124\,KB \\
$10{,}000$ & 100 & 400\,MB & 4\,MB \\
$100{,}000$ & 316 & 40\,GB & 126\,MB \\
$1{,}000{,}000$ & $1{,}000$ & 4\,TB & 4\,GB \\
\bottomrule
\end{tabular}
\end{center}

At $n = 10^5$, $S$ still fits on a single 192\,GB GPU under the trace reformulation (Table~\ref{tab:memory}); Phase~2 shards across four GPUs for throughput. At $n = 10^6$, $S$ is 4\,TB and requires multi-node distribution across 16 nodes (64 GPUs).

\paragraph{Factor model and mixing matrix.}
Both types share $A \in \mathbb{R}_+^{n \times (k+1)}$: column~0 is a common factor (mean loading $0.8$, Gaussian noise $\sigma = 0.15$ on all rows); columns $1,\dots,k$ are group factors with strong on-block loadings (uniform in $[0.6,1.8]$), off-block noise $\sigma = 0.05$, and block sizes $\pm 40\%$ around $n/k$. Entries are clamped to $\mathbb{R}_+$; the resulting $A$ has full column rank $k{+}1$ in every generated instance.

Before computing $S$, the rows of $A$ (and the corresponding ground-truth labels) are randomly permuted. Permuting $A$ early rather than the $n \times n$ output avoids an $O(n^2)$ random-access shuffle on the final matrix when $S$ is memory-mapped. Each matrix is stored with its permutation vector and ground-truth cluster labels.

\paragraph{TPDM construction.}
We generate $n_{\mathrm{sample}} = 40{,}000$ synthetic observations via $X = AZ + \varepsilon$, where $Z \in \mathbb{R}^{(k+1) \times n_{\mathrm{sample}}}$ has iid $\text{Pareto}(\alpha = 2)$ entries and $\varepsilon$ has iid $0.3 \cdot \text{Pareto}(\alpha = 2)$ entries. Each observation $X_j$ is mapped to polar coordinates $(r_j, \bm{w}_j) = (\|X_j\|, X_j / \|X_j\|)$. The $n_{\mathrm{exc}} = \lfloor q \cdot n_{\mathrm{sample}} \rfloor$ observations with largest radial norm $r_j$ are retained as exceedances ($q = 0.01$), and their angular components form the rows of $\Omega \in \mathbb{R}^{n_{\mathrm{exc}} \times n}$ (Section~\ref{estimation}). The empirical TPDM is $\hat{\Sigma} = (n / n_{\mathrm{exc}})\, \Omega^\top \Omega$, normalized to unit diagonal. By Corollary~\ref{cor}, under the noiseless factor model the population TPDM is $\Sigma = D^{-1/2}AA^\top D^{-1/2}$ with $D_{ii} = (AA^\top)_{ii}$ (recall $\alpha = 2$); with the same-index Pareto noise used here, this is the leading factor contribution to which the estimator converges as $n_{\mathrm{sample}} \to \infty$. Fixing $n_{\mathrm{sample}}$ (hence $n_{\mathrm{exc}} = 400$) makes large-$n$ TPDM hardness mix spectral flattening with a fixed exceedance budget.

\paragraph{Correlation construction.}
Using the same mixing matrix $A$ and $n_{\mathrm{sample}} = 40{,}000$, we generate $X = AZ + \varepsilon$ with $Z \sim \mathcal{N}(0, I_{k+1})$ and $\varepsilon \sim \mathcal{N}(0, I_n)$. The Pearson sample correlation is $C_{ij} = \hat{\mathrm{Cov}}(X_i, X_j) / (\hat\sigma_i \hat\sigma_j)$. Since $A \ge 0$ and $Z$, $\varepsilon$ are zero-mean Gaussian, the population covariance $AA^\top + I_n$ has non-negative off-diagonals, so sample negatives arise only from finite-sample noise. The SymNMF input is $S_{ij} = |C_{ij}|$, symmetric, non-negative, unit-diagonal, and valued in $[0, 1]$.

\paragraph{Rank choice.}
The generative rank is $k^* = \lfloor\sqrt{n}\rfloor + 1$; the factorization rank is $k = \lfloor\sqrt{n}\rfloor$, one below, so solvers must absorb the common factor into the group columns of $H$ (Section~\ref{snmf:matrices}). The fixed-rank stress test $k_{\mathrm{fix}} = 25$ is likewise defined there.

\paragraph{Scalability.}
At $n \le 10{,}000$, all intermediates fit in GPU RAM. At $n \ge 10^5$, the $n \times n_{\mathrm{sample}}$ intermediate is generated in row-chunks and the outer product ($\Omega^\top\Omega$ for TPDM, covariance tiles for correlation) is accumulated tile-by-tile into a memory-mapped file; no $n \times n$ intermediate is materialized. GPU-accelerated tiling uses deterministic per-block seeds for reproducibility.

\paragraph{Ledoit--Wolf shrinkage.}
When $n / n_{\mathrm{sample}} > 0.1$, the sample correlation is ill-conditioned. Ledoit--Wolf shrinkage \cite{ledoit2004wellconditioned} replaces it by $S_{\mathrm{shrunk}} = (1-\lambda_{\mathrm{LW}})S + \lambda_{\mathrm{LW}} \mu I$ with $\mu = \operatorname{tr}(S)/n$ and analytically estimated intensity~$\lambda_{\mathrm{LW}}$. The TPDM is not a sample covariance, so Ledoit--Wolf does not apply.

\subsection{Detailed Algorithm Descriptions}\label{algo_detail}

All solvers minimize $f(H) = \|S - HH^\top\|_F^2$ with $H \in \mathbb{R}_+^{n \times k}$, using the gradient $\nabla_H f = 4(H(H^\top H) - SH)$ from Section~\ref{snmf:loss}. In projected mode, every update is followed by the element-wise clamp $[X]_+ = \max(X, 0)$; in LogSpace mode, the optimizer acts on the unconstrained variable $\Theta$ with $H = \operatorname{softplus}(\Theta)$ (Section~\ref{snmf:logspace}). For each family we give the update rule and its per-iteration cost.

\paragraph{Family 1: Multiplicative / Mirror Descent.}
The dampened MU rule of \cite{he2011symmetric} is
\begin{equation*}
H \leftarrow \tfrac{1}{2}\,H \odot \bigl(\mathbf{1} + SH \oslash (H(H^\top H) + \epsilon)\bigr),
\end{equation*}
identical to~\eqref{eq:mu_update}, where $\odot$ and $\oslash$ denote element-wise multiplication and division, and $\epsilon$ is fp32 machine epsilon (${\approx}\,1.2\times 10^{-7}$), matching the implementation default.
MU preserves $H \ge 0$ by construction and guarantees monotonic decrease of an auxiliary function \cite{he2011symmetric}. We test four variants: random initialization, NNDSVD \cite{boutsidis2008svd} initialization, trace-norm regularization ($\lambda(\operatorname{tr}(S) - \|H\|_F^2)^2$), and mirror descent (exponentiated gradient $H \leftarrow H \odot \exp(-\eta\nabla_H f)$).

\paragraph{Family 2: Projected Gradient Methods.}
Projected Gradient Descent (PGD) with Armijo backtracking performs
\[
H^{(t+1)} = \bigl[H^{(t)} - \eta_t \nabla_H f(H^{(t)})\bigr]_+,
\]
where $\eta_t$ is reduced by factor $\beta = 0.5$ until the sufficient-decrease condition $f(H^{(t+1)}) \le f(H^{(t)}) - c\,\eta_t \|\nabla_H f\|_F^2$ holds ($c = 10^{-4}$).

Accelerated Proximal Gradient (APG / FISTA) \cite{beck2009fast} introduces the extrapolation step
\[
Y^{(t)} = H^{(t)} + \frac{\tau_t - 1}{\tau_{t+1}}(H^{(t)} - H^{(t-1)}), \qquad
\tau_{t+1} = \tfrac{1}{2}(1 + \sqrt{1 + 4\tau_t^2}),
\]
and applies PGD to $Y^{(t)}$. Adaptive restart \cite{odonoghue2015adaptive} resets $\tau$ to 1 whenever $f(H^{(t+1)}) > f(H^{(t)})$, preventing oscillation in the non-convex setting. Achieves $O(1/t^2)$ rate for convex problems. Heavy-ball momentum on projected PGD was tested and excluded (momentum fights the non-negativity clamp; see skipped solvers below).

\paragraph{Family 3: Adaptive First-Order Methods.}
All methods in this family maintain per-element running statistics of the gradient and use them to scale the learning rate. Letting $g_t = \nabla_H f(H^{(t)})$:

\emph{AdaGrad} \cite{duchi2011adagrad}: $G_t = G_{t-1} + g_t \odot g_t$, \;
$H^{(t+1)} = [H^{(t)} - \eta\, g_t \oslash (\sqrt{G_t} + \epsilon)]_+$.
The cumulative sum $G_t$ yields a monotonically decreasing effective learning rate. Benchmark: $\eta = 2.0$.

\emph{Piecewise AdaGrad}: between resets, identical to AdaGrad above. On objective stagnation (criterion~3, Section~\ref{snmf:convergence}), fire a reset only while the per-element projected gradient remains well above the KKT gate ($\lVert\nabla_{\mathrm{proj}}\rVert_F/(nk) > c_r\,\tau_g$, with $c_r \in \{3,10,30\}$) and a median-squared-gradient floor holds ($\operatorname{median}(g_t^2)>\epsilon$), then recalibrate:
$G_i \leftarrow \max\bigl(c \cdot g_{t,i}^2,\; c \cdot \mathrm{median}(g_t^2)\bigr)$.
The median floor prevents $G_i \to 0$ on near-zero-gradient entries; the gradient-margin gate escapes large-scale plateaus early and self-terminates near convergence, replacing an absolute $E_t$ cap that fails at $n=10^6$. After a reset, stagnation detection stays off until $\sum_i G_i$ doubles from its post-reset value. A subsequent reset fires only if the previous one produced measurable loss reduction; otherwise resets are permanently disabled. Benchmark: $\eta = 1.0$, $c = 10$.

\emph{RMSprop} \cite{hinton2012rmsprop}: $v_t = \alpha\, v_{t-1} + (1-\alpha)\, g_t \odot g_t$, \;
$m_t = \mu\, m_{t-1} + \eta\, g_t \oslash (\sqrt{v_t} + \epsilon)$, \;
$H^{(t+1)} = [H^{(t)} - m_t]_+$.
Benchmark: $\eta = 0.05$, $\alpha = 0.99$, $\mu = 0.5$.

\emph{Adam} \cite{kingma2014adam}: $m_t = \beta_1 m_{t-1} + (1-\beta_1) g_t$, \;
$v_t = \beta_2 v_{t-1} + (1-\beta_2) g_t \odot g_t$, \;
$\hat{m}_t = m_t / (1-\beta_1^t)$, \; $\hat{v}_t = v_t / (1-\beta_2^t)$, \;
$H^{(t+1)} = [H^{(t)} - \eta\, \hat{m}_t \oslash (\sqrt{\hat{v}_t} + \epsilon)]_+$.
Benchmark: $\eta = 0.01$, $\beta_1 = 0.9$, $\beta_2 = 0.999$.
NAdam \cite{dozat2016incorporating} replaces $\hat{m}_t$ with $\beta_1 \hat{m}_t + (1-\beta_1)g_t/(1-\beta_1^t)$ (Nesterov look-ahead). Adan \cite{xie2023adan} maintains three buffers with Nesterov-style gradient differences.

\emph{Row-Stochastic AdaGrad}: sample a row index set $I\subset[n]$ and form the exact partial gradient on those rows (cost $O(|I|\,nk)$). Plain sub-sampling leaves unsampled rows with stale $G_i$ and fails the KKT gate; \emph{Row-Stochastic SVRG} repairs this with a periodic full-gradient snapshot $\mu=\nabla_H f(H)$: fresh gradients on $I$, snapshot rows elsewhere, so every row of $H$ and of $G$ updates each iteration while retaining $O(|I|\,nk)$ cost between snapshots. Phase~1/2 anchor: $(\eta,\,|I|/n)=(2.0,\,0.5)$.

\emph{Entry-stochastic AdaGrad}: sample $m$ pairs $(i,j)$ and accumulate the per-entry contributions
\begin{equation*}
  g_i \mathrel{+}= 2\bigl(h_i^\top h_j - S_{ij}\bigr)\,h_j, \qquad
  g_j \mathrel{+}= 2\bigl(h_i^\top h_j - S_{ij}\bigr)\,h_i
\end{equation*}
(cost $O(mk)$; unbiased for $\nabla_H f$ up to a global factor absorbed by~$\eta$). Fixed small $m/n^2$ fails the full-gradient KKT gate (Section~\ref{snmf:perf}); Appendix~\ref{app:failed} records the variance-reduction attempts.

\emph{Block-SVRG \textsc{AdaptGrow}}: sample index sets $I,J$ with $|I|=|J|=\lceil\sqrt{\phi}\,n\rceil$ and form the dense $I\times J$ residual tile via cuBLAS ($R_{IJ}=H_I H_J^\top - S_{IJ}$), contributing $g_I \mathrel{+}= 2 R_{IJ} H_J$ and $g_J \mathrel{+}= 2 R_{IJ}^\top H_I$. Coverage is $\phi n^2$ entries per step. Grow $\phi\leftarrow\min(2\phi,1)$ when objective stagnation persists for a short window while $\lVert\nabla_{\mathrm{proj}}\rVert_F/(nk)$ remains above a multiple of~$\tau_g$ (norm-test style; Section~\ref{snmf:perf}); once $\phi=1$ the step is full-batch AdaGrad. An SVRG hybrid uses a full-gradient snapshot as the placeholder on rows outside $I\cup J$, which prevents per-row accumulator heterogeneity during the stochastic phase. Benchmark anchors: $\phi_0\in\{0.5,0.75\}$ with snapshot period ${\sim}1/\phi$.

\paragraph{Family 4: Block Coordinate Descent.}
Both methods use the two-factor formulation of \cite{kim2014algorithms}:
\[
\min_{H,W \ge 0} \|S - WH^\top\|_F^2 + \alpha \|H - W\|_F^2,
\]
where $\alpha \gg 1$ couples $W \approx H$. Each sub-problem is then a standard (asymmetric) NMF update.

\emph{HALS} \cite{kimpark2008}: Updates one column $h_j$ of $H$ at a time. Letting $R_j = S - \sum_{l \ne j} w_l h_l^\top$ denote the residual with column $j$ removed,
\[
h_j \leftarrow \Bigl[\frac{R_j^\top w_j + \alpha\, w_j}{\|w_j\|^2 + \alpha}\Bigr]_+,
\]
with $W$ held fixed at the previous $H$ for a full column sweep of~$H$, then refreshed ($W\leftarrow H$) on the next outer iteration. Incremental residual updates avoid recomputing $R_j$ from scratch. Per-iteration cost is $O(n^2 k)$. Benchmark: $\alpha \in \{10^2, 10^3, 10^4\}$.

\emph{ANLS} \cite{kim2014algorithms}: Solves the $H$-update as a non-negative least squares problem by forming
\[
H^\top \leftarrow \arg\min_{\tilde{H} \ge 0} \Bigl\|\begin{bmatrix} W \\ \sqrt{\alpha}\, I_k \end{bmatrix} \tilde{H} - \begin{bmatrix} S \\ \sqrt{\alpha}\, W^\top \end{bmatrix}\Bigr\|_F^2,
\]
requiring a $k \times k$ Cholesky factorization per column and $O(n^2)$ memory for the stacked system. Benchmark: $\alpha \in \{10^2, 10^3\}$.

\paragraph{Family 5: Second-Order Methods.}
\emph{Diagonal Newton} with Levenberg--Marquardt regularization: at each iteration, compute the diagonal of the Hessian $d_{ij} = [\nabla^2 f]_{ij,ij}$ and step
\[
H^{(t+1)} = \Bigl[H^{(t)} - \eta\, \frac{[\nabla_H f]_{ij}}{|d_{ij}| + \lambda_{\mathrm{reg}}}\Bigr]_+,
\]
with Armijo backtracking on $\eta$. The absolute value $|d_{ij}|$ handles indefinite Hessian entries from saddle points. Per-iteration cost: $O(n^2 k)$. Benchmark: $\eta = 0.1$, $\lambda_{\mathrm{reg}} = 0.5$.

\emph{L-BFGS} \cite{nocedal2006numerical}: Approximates the inverse Hessian via $m$ stored gradient pairs $(s_i, y_i)$ using the two-loop recursion. Each line-search evaluation costs $O(n^2 k)$, and the overall cost is $O(m \cdot n^2 k)$ per outer step. Benchmark: $m = 100$, with projected Armijo line search.

\emph{PGNCG} \cite{eswar2020pgncg}: Projected Gauss--Newton with truncated CG (basis of PLANC~\cite{eswar2021planc}). Each outer step solves a linearized least-squares subproblem in the free variables by CG; per-iteration cost remains $O(n^2 k)$ plus CG matvecs. Under our three-criterion gate it reaches only 11/18 runs and no $n = 10{,}000$ run (Section~\ref{snmf:perf}).

\paragraph{Family 6: Splitting Methods.}
\emph{ADMM}: We split non-negativity from the factorization via an auxiliary factor $W$ and scaled dual $U$, with consensus $H = W$ and $W \ge 0$. Each iteration is
\begin{align}
H^{(t+1)} &= \bigl(2SW^{(t)} + \rho(W^{(t)} - U^{(t)})\bigr)\bigl(2{W^{(t)}}^\top W^{(t)} + \rho\, I_k\bigr)^{-1}, \label{eq:admm_H}\\
W^{(t+1)} &= \bigl[H^{(t+1)} + U^{(t)}\bigr]_+, \label{eq:admm_W}\\
U^{(t+1)} &= U^{(t)} + H^{(t+1)} - W^{(t+1)}, \label{eq:admm_U}
\end{align}
i.e.\ an unconstrained least-squares $H$-update for the bilinear surrogate $\|S - HW^\top\|_F^2$ with consensus penalty, followed by the Euclidean projection of $H+U$ onto $\mathbb{R}_+^{n \times k}$. Convergence and reported solutions are tracked on the feasible factor~$W$. The $H$-update inverts a single $k \times k$ matrix (cost $O(k^3)$) and one $O(n^2 k)$ product $SW$. Benchmark: $\rho \in \{50, 500\}$.

\emph{Randomized LAI-SymNMF} \cite{hayashi2024randomized}: Computes a rank-$k$ approximate eigendecomposition $S \approx U\Lambda U^\top$ by a randomized range finder (cost $O(n^2 k)$) and then runs two-block HALS on $n \times k$ factors $(H,W)$ with the cheap matvec $U(\Lambda \odot (U^\top H))$ in place of $SH$, using the same $\alpha$-coupling as Family~4. Optional refinement on the full $S$ was tested and disabled in the reported runs ($0$ refine iterations). Benchmark: $\alpha = 100$.

\paragraph{Family 7: Deep Unfolding.}
\emph{SymNMF-Net} \cite{li2022symncf}: Unfolds $T$ iterative updates into a trainable network with learnable parameters $\{P_t, \lambda\}$. Each block consists of an inversion layer $(H^\top H + \lambda I)^{-1}$, a linear layer with parameter $P_t$, and a ReLU activation for non-negativity. The network is trained from scratch for each input~$S$ using Adam on the reconstruction loss, yielding $T$ blocks of cost $O(n^2 k)$ each per training iteration. Benchmark: $T = 5$, $\eta = 0.1$.

\paragraph{Skipped solvers and documented failure modes.}
Two solver configurations were excluded from the benchmark after initial testing revealed structural (not hyperparameter) failures; a third, ANLS, is benchmarked but does not scale:
\begin{itemize}[leftmargin=2em,topsep=2pt,itemsep=1pt]
\item \textbf{PGD + Momentum} (projected): Momentum carries $H$ past zero; the projection clamps it back to $0$ while the momentum buffer retains the negative velocity. This feedback loop drives $H \to 0$ within ${\sim}10$ iterations ($E_t = 1.0$).
\item \textbf{Shampoo} \cite{gupta2018shampoo}: Tested via \texttt{torch\_optimizer.Shampoo} in projected and LogSpace modes. For $H\in\mathbb{R}^{n\times k}$ the left Kronecker factor is $n\times n$, so each matrix-root update is $O(n^3)$ and already impractical at our scales. Early runs also failed (projected: NaN in the matrix root by $n\ge 100$; LogSpace: no useful progress on $E_t$). Excluded on cost; the failures are secondary.
\item \textbf{ANLS} converges at $n \le 1{,}000$ (Section~\ref{results}) but is impractical beyond it: the stacked least-squares system carries the $n \times n$ block $S$ on its right-hand side, costing $O(n^2)$ additional memory and ${\sim}10$\,s per outer iteration at $n = 10{,}000$, where it exhausts the iteration budget without converging, leaving it uncompetitive with adaptive methods that finish in $< 1$\,s total.
\end{itemize}

\subsection{Sub-quadratic SymNMF: failed approaches and structural constraints}\label{app:failed}

This appendix records the exploratory analysis behind Block-SVRG \textsc{AdaptGrow} (Section~\ref{snmf:perf}): why fixed-fraction entry sampling fails the KKT gate under variance reduction, why low-rank sketching is not a substitute, and why adaptive batch growth is specific to entry-stochastic AdaGrad.

\paragraph{Variance reduction at fixed entry fraction.}
Section~\ref{snmf:perf} frames the fixed-fraction failure as a criterion mismatch: updates use sparse entry gradients, while certification requires the exact full gradient below~$\tau_g$. The per-iteration entry budget therefore caps the achievable KKT residual at a fixed fraction $m/n^2$, independent of how that budget is aggregated. We tested three variance-reduction strategies; none closes the gap.

\emph{SVRG variance reduction.} By analogy with the row-sampled SVRG variant, we periodically compute a full-gradient snapshot $\mu = \nabla_H f(H)$ and construct a hybrid gradient: fresh entry-sampled estimates on touched rows, stale $\mu$ elsewhere \cite{johnson2013svrg}. Near a stationary point $H^*$ the stale components carry systematic bias proportional to $\|H - H_{\mathrm{snap}}\|$. For the oscillation amplitude to fall below the KKT threshold, snapshots must be refreshed every $10$--$100$ iterations at $O(n^2 k)$ cost each, at which point the amortized cost approaches full-batch AdaGrad.

\emph{Polyak--Ruppert iterate averaging.} Suffix averaging \cite{polyak1992acceleration, ruppert1988efficient} maintains $\bar{H}_T = (1/T') \sum_{t=T-T'}^{T} H^{(t)}$ after a warm-up triggered at $E_t < 0.1$. Averaging reduces objective variance, but the projected-gradient norm of the averaged iterate $\bar H$ remains far above threshold because $\nabla f(\bar H) \ne \overline{\nabla f(H^{(t)})}$ for non-linear $f$.

\emph{Dual averaging.} The ADAGRAD-RDA framework of \cite{xiao2010dual} accumulates gradients in a dual variable $Z_t = \sum_{s=1}^{t} g_s$ and recovers the primal iterate via a proximal step:
\begin{equation}\label{eq:dual_avg}
  H^{(t)} = \arg\min_{H \ge 0} \Bigl\{ \langle Z_t, H \rangle
    + \frac{1}{\eta}\sum_{i,j}\sqrt{G_{t,ij}}\,H_{ij}^2
    + \frac{\delta\sqrt{t}}{2}\|H\|_F^2 \Bigr\},
\end{equation}
where $G_t$ is the AdaGrad accumulator and $\delta\sqrt{t}$ prevents divergence. The proximal step enforces $H \ge 0$ without the gradient--projection conflict of primal methods. However, $\nabla_H f(0) = 0$, so $H = 0$ is a critical point of the SymNMF objective, and the proximal recovery initializes $H$ near zero. A primal-AdaGrad warm-up escapes this neighborhood before transitioning to dual accumulation, but the KKT residual remains large under the same fixed entry budget.

\paragraph{Why sketching is not a natural alternative.}
A tempting sub-quadratic strategy is to replace $S$ with a rank-$r$ spectral truncation $\hat S = U\Lambda U^\top$ ($r \ll n$) and solve $\min_{H \ge 0} \|\hat S - HH^\top\|_F^2$ at cost $O(nrk)$ per iteration \cite{halko2011finding, woodruff2014sketching}. The gradient error is bias ($\|\nabla f_{\hat S} - \nabla f_S\| \le O(\|S - \hat S\|_F)$) rather than variance.
However, the dependence matrices here are not well approximated at small~$r$: correlation has a Marchenko--Pastur bulk with slowly decaying eigenvalues (Table~\ref{tab:spectral}), and the empirical TPDM, though rank-deficient with $\operatorname{rank}\le n_{\mathrm{exc}}=400$, concentrates its remaining spectrum on a flat plateau. A modest-rank sketch therefore leaves $\|S - \hat S\|_F$ above the KKT scale; raising $r$ to $O(\sqrt{n})$ only mildly subquadratizes the cost ($O(n^{3/2}k)$). The downstream user also typically needs the factorization of the true~$S$, not a surrogate.

\paragraph{Why growing-batch is specific to entry-stochastic AdaGrad.}
Two natural questions arise: (1)~can adaptive batch growth be applied to row-sampled stochastic AdaGrad, and (2)~can stochastic sampling be combined with RMSprop? Both answers are negative for structural reasons.

\emph{Row-sampled AdaGrad does not benefit from batch growth.} Naive row sampling fails the KKT gate through a mix of gradient noise and uneven accumulator growth: infrequently sampled rows develop stale $G_i$ and mismatched per-row rates $\eta/\sqrt{G_i}$. Row-SVRG already closes that gap at fixed $|I|/n$ by writing a hybrid gradient on \emph{every} row each step (fresh on~$I$, snapshot elsewhere), so $G$ grows uniformly without growing the batch. Batch growth would only erase the per-step saving that makes row-SVRG useful at Phase~2.

\emph{RMSprop's EMA is not missing-data stable under the same masking.} Write an unsampled coordinate as $g_{t,i}=0$ (the natural input to a dense adaptive optimizer). AdaGrad is invariant: $G_{t,i}=G_{t-1,i}+0$ freezes the rate. An EMA is not: $v_{t,i}=\alpha v_{t-1,i}+(1-\alpha)\,0=\alpha v_{t-1,i}$ decays on every miss, so the next touch is overscaled (at $\alpha=0.99$, roughly $\alpha^{-\Delta/2}$ after $\Delta$ misses). Skipping unsampled coordinates instead of writing zeros leaves both accumulators unchanged on those rows, but standard dense RMSprop implementations apply the EMA to the full tensor and therefore see the zeros. AdaGrad's additive update is the same under either convention; the EMA is not.

SVRG does not repair this for RMSprop. The hybrid step feeds a nonzero placeholder $\mu_i$ into every row each iteration, so the EMA is continuously driven by stale $\mu_i^2$ between snapshots. Under entry sampling the short EMA window also retains sampling variance that AdaGrad's $1/t$ average damps. After $\phi\to 1$, AdaGrad can keep accumulating on top of the stochastic history; an EMA still has to forget that history over its effective window.

In summary, the four stochastic variants in this work (row-sampled, entry-sampled, SVRG, \textsc{AdaptGrow}) are built on AdaGrad's monotone accumulator because it is stable under partial observation. RMSprop's exponential forgetting helps in the full-batch regime (Section~\ref{snmf:perf}) but rules it out here.

\paragraph{Dense-phase reset does not repair block-sampling heterogeneity.}
Block-structured sampling (Section~\ref{snmf:perf}) makes AdaGrad's accumulator $G_i$ heterogeneous across rows; the recommended fix is the SVRG hybrid snapshot, which smooths the per-row update density preemptively. We also tested a corrective alternative: once the batch is full ($\phi = 1$), on the next stagnation event recalibrate $G_i \leftarrow \max(c\,g_i^2,\, c\,\mathrm{median}(g^2))$ from a single fresh dense gradient (the Piecewise AdaGrad rule, applied only after the batch is full so the reset never bakes stochastic-phase heterogeneity into~$G$), with a cooldown until $G$ doubles and auto-disable if no objective progress follows. This does not resolve the heterogeneity: the reset fires on transient stagnation and disrupts the per-row scaling rather than repairing it. Stacking the reset on top of the SVRG hybrid is likewise counterproductive: it disrupts an already-coherent~$G$. The dense-phase reset is therefore at best a memory-constrained fallback when SVRG snapshots are infeasible, and should not be combined with the SVRG snapshot.

\clearpage

\bibliographystyle{unsrt} 
\bibliography{biblio}     

\end{document}